%% file: main.tex
\PassOptionsToPackage{pagebackref}{hyperref}
\PassOptionsToPackage{capitalize}{cleveref}
\PassOptionsToPackage{dvipsnames}{xcolor}
\PassOptionsToPackage{ruled,linesnumbered,vlined}{algorithm2e}
\documentclass[final, nosubfloats, cleveref]{l4dc2025}

\let\qed\jmlrQED
\usepackage[utf8]{inputenc} % allow utf-8 input
\usepackage[T1]{fontenc}    % use 8-bit T1 fonts
\usepackage{xurl}
\usepackage{booktabs}       % professional-quality tables
\usepackage{amsfonts}       % blackboard math symbols
\usepackage{nicefrac}       % compact symbols for 1/2, etc.
\usepackage{microtype}      % microtypography
\usepackage{float}

\usepackage{mathtools}
\usepackage{thmtools}

\usepackage{bm}
\usepackage[mathcal]{euscript}
\usepackage{physics}
\usepackage{siunitx}
\usepackage[shortlabels]{enumitem}
\usepackage{aliascnt}
\usepackage{accents}
\usepackage{tikz}
\usepackage{xkcdcolors}
\usepackage{tablefootnote}
\usepackage{caption}
\usepackage{subcaption}
\usepackage{wrapfig}
\usepackage[export]{adjustbox}
\usepackage{titlesec}

\newcommand\myshade{60}
\colorlet{mylinkcolor}{YellowOrange}
\colorlet{mycitecolor}{MidnightBlue}
\colorlet{myurlcolor}{violet}
\definecolor{goodcolor}{HTML}{0ab246}

\hypersetup{
    colorlinks=true,
    linkcolor=goodcolor!\myshade!black,
    citecolor=goodcolor!\myshade!black,    
    urlcolor=goodcolor!\myshade!black,
}

\usetikzlibrary{positioning}
\usetikzlibrary{calc}

\titlespacing*{\section}
  {0pt}   % Left margin
  {10pt}  % Space before
  {5pt}   % Space after
\declaretheorem[sibling=definition]{assumption}
\crefname{assumption}{Assumption}{Assumptions}
\makeatletter
\AddToHook{cmd/appendix/before}{\def\cref@section@alias{appendix}}
\makeatother

\SetAlCapHSkip{0pt}
\titlespacing*{\paragraph}{0pt}{0.5ex}{0.5em}

\DeclareMathOperator*{\argmin}{arg\,min}
\newcommand{\smax}{\bar\sigma}
\newcommand{\smin}{\underaccent{\bar}{\sigma}}
\allowdisplaybreaks
\crefformat{assa}{Assumption~#2#1a#3}
\crefformat{assb}{Assumption~#2#1b#3}
\colorlet{onno}{xkcdPeriwinkleBlue}
\colorlet{claire}{xkcdPaleGreen}

\crefname{equation}{}{}
\crefname{lemma}{Lemma}{Lemmas}

\let\norm\undefined
\DeclarePairedDelimiter\norm{\lVert}{\rVert}

\newtoggle{noappendix}
\togglefalse{noappendix}

\title[A Pontryagin Perspective on Reinforcement Learning]{A Pontryagin Perspective on Reinforcement Learning}
\usepackage{times}
\coltauthor{
  \Name{Onno Eberhard}{\normalfont\textsuperscript{12}} \Email{oeberhard@tue.mpg.de}\\
  \Name{Claire Vernade}{\normalfont\textsuperscript{2}} \Email{claire.vernade@uni-tuebingen.de}\\
  \Name{Michael Muehlebach}{\normalfont\textsuperscript{1}} \Email{michaelm@tue.mpg.de}\\
  \addr {\normalfont\textsuperscript{1}}Max Planck Institute for Intelligent Systems, Tübingen, Germany\\{\normalfont\textsuperscript{2}}University of Tübingen}
\begin{document}

\maketitle

\begin{abstract}%
  Reinforcement learning has traditionally focused on learning state-dependent policies to solve optimal control problems in a \emph{closed-loop} fashion.
  In this work, we introduce the paradigm of \emph{open-loop reinforcement learning} where a fixed action sequence is learned instead.
  We present three new algorithms: one robust model-based method and two sample-efficient model-free methods.
  Rather than basing our algorithms on Bellman's equation from dynamic programming, our work builds on \emph{Pontryagin's principle} from the theory of open-loop optimal control.
  We provide convergence guarantees and evaluate all methods empirically on a pendulum swing-up task, as well as on two high-dimensional MuJoCo tasks, significantly outperforming existing baselines.\looseness=-1
\end{abstract}

\begin{keywords}%
  Reinforcement learning, Open-loop control, Non-convex optimization
\end{keywords}

\section{Introduction}
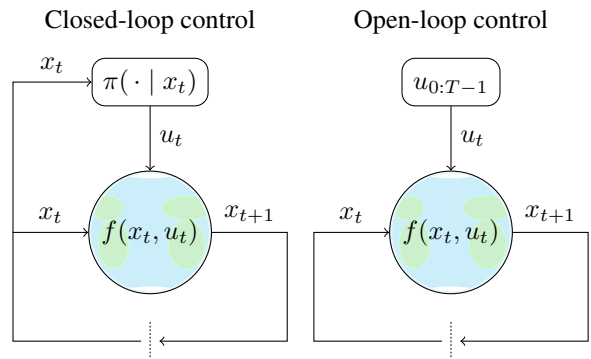
\begin{wrapfigure}{r}{0.6\textwidth}
  \vspace{-4em}
  \begin{center}
    \resizebox{0.95\linewidth}{!}{\input{loop.tikz}}
  \end{center}
  \vspace{-1em}
  \caption{Comparison of closed-loop (left) and open-loop (right) control. In closed-loop RL, the goal is to learn a policy $\pi$. In open-loop RL, a fixed sequence of actions $u_{0:T-1}$ is learned instead, with $u_t$ independent of the states $x_{0:t}$.}
  \label{fig:loop}
  \vspace{-1.2em}
\end{wrapfigure}

Reinforcement learning (RL) refers to ``the optimal control of incompletely-known Markov decision processes'' \citep[p.~2]{sutton2018reinforcement}.
It has traditionally focused on applying dynamic programming algorithms, such as value iteration or policy iteration, to situations where the environment is unknown.
These methods solve optimal control problems in a closed-loop fashion by learning feedback policies, which map states ($x_t$) to actions ($u_t$).
In contrast, this work introduces the paradigm of \emph{open-loop reinforcement learning} (OLRL), in which fixed action sequences $u_{0:T-1}$, over a horizon $T$, are learned instead.
The closed-loop and open-loop control paradigms are illustrated in \cref{fig:loop}.

An open-loop controller receives no observations from its environment.
This makes it impossible to react to unpredictable events, which is essential in many problems, particularly those with stochastic or unstable dynamics.
For this reason, RL research has historically focused exclusively on closed-loop control.
However, many environments are perfectly predictable.
Consider the classic example of swinging up an inverted pendulum.
If there are no disturbances, then this task can be solved flawlessly without feedback (as we demonstrate in \cref{sec:pendulum}).
Where open-loop control is viable, it brings considerable benefits.
As there is no need for sensors, it is generally much cheaper than closed-loop control.
It can also operate at much higher frequencies, since there is no bandwidth bottleneck due to sensor delays or computational processing of measurements.
% Third, if the environment changes in an unforeseen way, a feedback policy may produce unpredictable behavior, whereas an open-loop solution, which cannot detect environmental changes, is unaffected.
Importantly, the open-loop optimal control problem is much simpler, as it only involves optimizing an action sequence (finding one action per time step).
In contrast, closed-loop optimal control involves optimizing a policy (finding one action for each state of the system), which can be considerably more expensive. % in the finite-horizon setting.
In this way, open-loop control circumvents the curse of dimensionality without requiring function approximation.\looseness=-1

For these reasons, open-loop control is widely used in practice \citep{diehl2006,Zundert2018,sferrazza2020}, and there exists a large body of literature on the theory of open-loop optimal control \citep{pontryagin-1962-dynamic}.
However, the setting of incompletely-known dynamics has received only little attention.
In this work, we introduce a family of three new open-loop RL algorithms by adapting the existing theory to this setting.
Whereas closed-loop RL is largely based on approximating the Bellman equation, the central equation of dynamic programming, we base our algorithms on approximations of \emph{Pontryagin's principle}, the central equation of open-loop optimal control.
We first introduce a model-based method whose convergence we prove to be robust to modeling errors.
This is a novel and non-standard result which depends on a careful analysis of the algorithm.
We then extend this procedure to settings with completely unknown dynamics and propose two fully online model-free methods.
Finally, we empirically demonstrate the robustness and sample efficiency of our methods on an inverted pendulum swing-up task and on two complex MuJoCo tasks.
\looseness=-1

\paragraph{Related work.}
Our work is inspired by numerical optimal control theory \citep{betts2010practical,geering2007}, which deals with the numerical solution of trajectory optimization problems.
Whereas existing methods assume that the dynamics are known, our algorithms only require an approximate model (model-based OLRL) or no model at all (model-free OLRL), and rely on a simulator to provide samples.
% We also see connections between our approach and itarative learning control, as well as to recent advances in deep learning.
An in-depth review of related work can be found in \cref{app:related-work}.
\iftoggle{noappendix}{
All appendices are contained in the full version of this paper \citep{eberhard-2024-pontryagin}.
}

\section{Background}\label{sec:background}
We consider a reinforcement learning setup with continuous state and action spaces $\mathcal X \subset \mathbb R^D$ and $\mathcal U \subset \mathbb R^K$.
Each episode lasts $T$ steps, starts in the fixed initial state $x_0$, and follows the deterministic dynamics $f: \mathcal X \times \mathcal U \to \mathcal X$, such that $x_{t+1} = f(x_t, u_t)$ for all times $t \in [T-1]_0$.\footnote{For $n \in \mathbb N$, we write $[n] \doteq \{1, 2, \dots, n\}$ and $[n]_0 \doteq \{0, 1, \dots, n\}$. Unless explicitly mentioned, all time-dependent equations hold for all $t \in [T-1]_0$.}
After every transition, a deterministic reward $r(x_t, u_t) \in \mathbb R$ is received, and at the end of an episode, an additional terminal reward $r_T(x_T) \in \mathbb R$ is computed.
The value of state $x_t$ at time $t$ is the sum of future rewards\looseness=-1
\begingroup
\abovedisplayskip=4pt
\belowdisplayskip=4pt
\begin{equation*}
    v_t(x_t; u_{t:T-1}) \doteq \sum_{\tau=t}^{T-1} r(x_\tau, u_\tau) + r_T(x_T) = r(x_t, u_t) + v_{t+1}\{f(x_t, u_t); u_{t+1:T-1}\}\text,
\end{equation*}
\endgroup
where we defined $v_T$ as the terminal reward function $r_T$.
Our goal is to find a sequence of actions $u_{0:T-1} \in \mathcal U^T$ maximizing the total sum of rewards $J(u_{0:T-1}) \doteq v_0(x_0; u_{0:T-1})$.
We will tackle this trajectory optimization problem using gradient ascent.
Although our goal is to learn an open-loop controller (an action sequence), we assume that the state is fully observed during the training process.

\paragraph{Pontryagin's principle.}
The gradient of the objective function $J$ with respect to the action $u_t$ is
\begin{equation}\label{eq:gradient}
  \nabla_{u_t}J(u_{0:T-1}) = \nabla_u r(x_t, u_t)
  + \nabla_u f(x_t, u_t) \underbrace{\nabla_x v_{t+1}(x_{t+1}; u_{t+1:T-1})}_{\lambda_{t + 1} \in \mathbb R^D}\text,
\end{equation}
where the terms of $J$ related to the earlier time steps $\tau \in [t - 1]_0$ vanish, as they do not depend on $u_t$.
We denote Jacobians as $(\nabla_y f)_{i, j} \doteq \frac{\partial f_j}{\partial y_i}$.
The \emph{costates} $\lambda_{1:T}$ are defined as the gradients of the value function along the given trajectory.
They can be computed through a backward recursion:
\begin{align}
  \lambda_T &\doteq \nabla v_T(x_T) = \nabla r_T(x_T)\label{eq:costate1}\\
  \lambda_t &\doteq \nabla_x v_t(x_t; u_{t:T-1}) = \nabla_x r(x_t, u_t) + \nabla_x f(x_t, u_t) \lambda_{t + 1}.\label{eq:costate2}
\end{align}
The gradient \cref{eq:gradient} of the objective function can thus be obtained by means of one forward pass through the dynamics $f$ (a rollout), yielding the states $x_{0:T}$, and one backward pass through \cref{eq:costate1,eq:costate2}, yielding the costates $\lambda_{1:T}$.
The stationarity condition arising from setting \cref{eq:gradient} to zero, where the costates are computed from \cref{eq:costate1,eq:costate2}, is known as \emph{Pontryagin's principle}.
(Pontryagin's principle in fact goes much further than this, as it generalizes to infinite-dimensional and constrained settings.)
We re-derive \cref{eq:gradient,eq:costate1,eq:costate2} using the method of Lagrange multipliers in \cref{app:lagrange}.

\section{Method}\label{sec:method}
If the dynamics are known, then the trajectory can be optimized by performing gradient ascent with the gradients computed according to Pontryagin's equations \labelcref{eq:gradient,eq:costate1,eq:costate2}.
In this work, we adapt this idea to the domain of reinforcement learning, where the dynamics are unknown.
In RL, we are able to interact with the environment, so the forward pass through the dynamics $f$ is not an issue.
However, the gradient computation according to Pontryagin's principle requires the Jacobians $\nabla_x f_t \doteq \nabla_x f(x_t, u_t)$ and $\nabla_u f_t \doteq \nabla_u f(x_t, u_t)$ of the unknown dynamics.
In our methods, which follow the structure of \cref{alg:pontryagin}, we therefore replace these Jacobians by estimates $A_t \simeq \nabla_x f_t$ and $B_t \simeq \nabla_u f_t$.
Before discussing concrete methods for open-loop RL, whose main concern is the construction of appropriate estimates $A_t$ and $B_t$, we first show that replacing $\nabla_x f_t$ and $\nabla_u f_t$ in this way is sensible.
In particular, we show that, under certain assumptions on the accuracy of $A_t$ and $B_t$, \cref{alg:pontryagin} converges to an unbiased local optimum of the true objective $J$.
In the following sections we then discuss model-based and model-free open-loop RL methods.\looseness=-1

% We now show that, given sufficiently accurate estimates $A_t$ and $B_t$, this algorithm still converges to a local maximum of the objective $J$.
\subsection{Convergence of \texorpdfstring{\cref{alg:pontryagin}}{Algorithm~\ref{alg:pontryagin}}}\label{sec:theorem}
\begin{wrapfigure}[13]{r}{0.58\textwidth}
  \vspace{-4.2em}\hfill
  \begin{minipage}{.95\linewidth}
  \begin{algorithm2e}[H]
    \caption{Open-loop reinforcement learning}\label{alg:pontryagin}
    \KwIn{Optimization steps $N \in \mathbb N$, step size $\eta > 0$}
    Initialize $u_{0:T-1}$ (initial action sequence)\\
    \For{$k = 1, 2, \dots, N$}{
        % \tcp{Forward pass}
        $x_{0:T} \gets \operatorname{rollout}(u_{0:T-1})$\hspace{1em}\tcp{Forw.\ pass}\label{line:forward}
        \BlankLine
        \tcp{Backward pass}
        $\tilde\lambda_T \gets \nabla r_T(x_T)$\\
        \For{$t = T-1, T-2, \dots, 0$}{
            \tcp{Jacobian estimation}
            $A_t, B_t \simeq \nabla_x f(x_t, u_t), \nabla_u f(x_t, u_t)$\label{line:jacobians}\\
            \BlankLine
            \tcp{Pontryagin update}
            $\tilde\lambda_t \gets \nabla_x r(x_t, u_t) + A_t\tilde\lambda_{t+1}$\\
            $g_t \gets \nabla_u r(x_t, u_t) + B_t\tilde\lambda_{t+1}$\\
            $u_t \gets u_t + \eta g_t$\hspace{1em}\tcp{Grad.\ ascent}
        }
    }
  \end{algorithm2e}
  \end{minipage}
\end{wrapfigure}
% Before discussing concrete methods for open-loop RL, whose main concern is the construction of appropriate estimates $A_t$ and $B_t$, we first show that replacing $\nabla_x f_t$ and $\nabla_u f_t$ with such estimates is a good idea.
% In particular, we show that, under certain assumptions on the accuracy of $A_t$ and $B_t$, \cref{alg:pontryagin} converges to an unbiased local optimum of the true objective $J$.
Our convergence result relies on the following three assumptions.
\begin{assumption}\label{ass:reward}
    All rewards are encoded in the terminal reward $r_T$. In other words, $r(x, u) = 0$ for all \(x \in \mathcal X\) and \(u \in \mathcal U\).
\end{assumption}
This assumption is without loss of generality, since we can augment the state $x_t$ by a single real variable $\rho_t$ that captures the sum of the running rewards (i.e., $\rho_0 = 0$ and $\rho_{t+1} = \rho_t + r(x_t, u_t)$).
An equivalent setup that satisfies \cref{ass:reward} is then obtained by defining a new terminal reward function $r'_T(x_T, \rho_T) \doteq r_T(x_T) + \rho_T$ and setting the running rewards $r'$ to zero.\looseness=-1

\begin{assumption}\label{ass:error}\label[assa]{ass:xerror}\label[assb]{ass:uerror}
    There exist constants $\gamma, \zeta > 0$ with $\gamma + \zeta + \gamma\zeta < 1$ such that for any trajectory $(u_{0:T-1}, x_{0:T})$ encountered by \cref{alg:pontryagin}, the following properties hold for all $t \in [T-1]_0$:
    \begin{enumerate}[(a),topsep=5pt,parsep=5pt]
        \item The error of $A_{t+s}$ is bounded, for all $s \in [T - t]$, in the following way:
        \begin{equation*}
            \norm{A_{t+s} - \nabla_x f_{t+s}} \leq \\\frac{\gamma}{3^s} \frac{\smin(\nabla_u f_t)}{\smax(\nabla_u f_t)}\left\{\prod_{i=1}^{s-1}\frac{\smin(\nabla_x f_{t+i})}{\smax(\nabla_x f_{t+i})}\right\}\smin(\nabla_x f_{t+s}).
        \end{equation*}
        \item The error of $B_t$ is bounded in the following way: $\norm{B_t - \nabla_u f_t} \leq \zeta \smin(\nabla_u f_{t})$.
    \end{enumerate}
    Here, $\smin(A)$ and $\smax(A)$ denote the minimum and maximum singular value of $A$, and $\norm{A} \doteq \smax(A)$.\vspace{-0.5em}
\end{assumption}
This assumption restricts the errors of the estimates $A_t$ and $B_t$ that are used in place of the true Jacobians $\nabla_x f_t$ and $\nabla_u f_t$ in \cref{alg:pontryagin}.
Although the use of the true system for collecting rollouts prevents a buildup of error in the forward pass, any error in the approximate costate $\tilde\lambda_t$ can still be amplified by the Jacobian estimates of earlier time steps, $A_\tau$ for $\tau \in [t - 1]$, during the backward pass.
Thus, to ensure convergence to a stationary point of the objective function $J$, the errors of these estimates need to be small.
This is particularly important for $t$ close to $T$, as these errors will be amplified over more time steps.
\Cref{ass:error} provides a quantitative characterization of this intuition.\looseness=-1\vspace{-0.5em}
\begin{assumption}\label{ass:smooth}
     There exists a constant $L > 0$ such that, for all action sequences $u_{0:T-1}^A, u_{0:T-1}^B \in \mathcal U^T$ and all times $t \in [T - 1]_0$, $\norm{\nabla_{u_t} J(u_{0:T-1}^A) - \nabla_{u_t} J(u_{0:T-1}^B)} \leq L \norm{u_t^A - u_t^B}$.
\end{assumption}\vspace{-0.5em}
This final assumption states that the objective function $J$ is $L$-smooth with respect to the action $u_t$ at each time step $t \in [T-1]_0$, which is a standard assumption in nonconvex optimization.
This implies that the dynamics $f$ are smooth as well.
We are now ready to state the convergence result.\vspace{-0.5em}
\begin{theorem}\label{thm:main}
Suppose \cref{ass:reward,ass:error,ass:smooth} hold with $\gamma$, $\zeta$, and $L$.
Let $\mu \doteq 1 - \gamma - \zeta - \gamma\zeta$ and $\nu \doteq 1 + \gamma + \zeta + \gamma\zeta$.
If the step size $\eta$ is chosen small enough such that $\alpha \doteq \mu - \frac{1}{2}\eta L \nu^2$ is positive, then the iterates $(u_{0:T-1}^{(k)})_{k=0}^{N - 1}$ of \cref{alg:pontryagin} satisfy, for all $N \in \mathbb N$ and $t \in [T-1]_0$,
\begin{equation*}
    \frac{1}{N}\sum_{k=0}^{N - 1} \norm{\nabla_{u_t} J(u_{0:T-1}^{(k)})}^2 \leq \frac{J^\star - J(u_{0:T-1}^{(0)})}{\alpha \eta N}\text,
\end{equation*}
where $J^\star \doteq \sup_{u\in \mathcal U^T}J(u)$ is the optimal value of the initial state.\vspace{-0.5em}
\end{theorem}
\begin{proof}
    See \cref{app:proof}.
    The proof depends on a novel intricate analysis of the backpropagation procedure in the case of an accurate forward pass and an inaccurate backward pass. This technique may also be applicable to other (non-control) domains, as described in \cref{app:related-work}.
\end{proof}
\vspace{-1em}
\subsection{Model-based open-loop RL}\label{sec:model-based}
The most direct way to approximate the Jacobians  $\nabla_x f_t$ and $\nabla_u f_t$ is by using a (learned or manually designed) differentiable model $\tilde f: \mathcal X \times \mathcal U \to \mathcal X$ of the dynamics $f$ and setting $A_t = \nabla_x \tilde f(x_t, u_t)$ and $B_t = \nabla_u \tilde f(x_t, u_t)$ in \cref{line:jacobians} of \cref{alg:pontryagin}.
\Cref{thm:main} guarantees that this \emph{model-based} open-loop RL method (see \cref{alg:model-based}) is robust to a certain amount of modeling error.
% In contrast to this, consider the more naive \emph{planning} method of using the model to obtain an imagined rollout $\tilde x_{0:T}$ and computing the gradient by differentiating 
In contrast to this, consider the more naive method of using the model to directly obtain a gradient by differentiating\looseness=-1
\begin{equation*}
  J(u_{0:T-1}) \simeq r(x_0, u_0) + r\{\underbrace{\tilde f(x_0, u_0)}_{\tilde x_1}, u_1\} + \dots + r_T\{\underbrace{\tilde f(\tilde f(\cdots \tilde f(\tilde f(x_0, u_0), u_1)\cdots), u_{T-1})}_{\tilde x_T}\}
\end{equation*}
with respect to the actions $u_{0:T-1}$ using the backpropagation algorithm.
In \cref{app:backprop}, we show that this \emph{planning} approach is exactly equivalent to an approximation of \cref{alg:pontryagin} where, in addition to setting $A_t = \nabla_x \tilde f(x_t, u_t)$ and $B_t = \nabla_u \tilde f(x_t, u_t)$, the forward pass of \cref{line:forward} is replaced by the imagined forward pass $\tilde x_{0:T}$ through the model $\tilde f$.
In \cref{sec:experiments}, we empirically demonstrate that this planning method, whose convergence is not guaranteed by \cref{thm:main}, is much less robust to modeling errors than the open-loop RL approach.
Note that neither method is related to \emph{model-predictive control} (MPC), which relies on measurements to re-plan at every step.
MPC is a closed-loop method that solves a fundamentally different problem from the one we address in this work.

\subsection{Model-free on-trajectory open-loop RL}\label{sec:on-policy}
Access to a reasonably accurate model may not always be feasible, and as \cref{alg:pontryagin} only requires the Jacobians of the dynamics along the current trajectory, a global model is also not necessary.
In the following two sections, we propose two methods that directly estimate the Jacobians $\nabla_x f_t$ and $\nabla_u f_t$ from rollouts in the environment.
We call these methods \emph{model-free}, as the estimated Jacobians are only valid along the current trajectory, and thus cannot be used for planning.

Our goal is to estimate the Jacobians $\nabla_x f(\bar x_t, \bar u_t)$ and $\nabla_u f(\bar x_t, \bar u_t)$ that lie along the trajectory induced by the action sequence $\bar u_{0:T-1}$.
These Jacobians measure how the next state ($\bar x_{t+1}$) changes if the current state or action ($\bar x_t$, $\bar u_t$) are slightly perturbed.
More formally, the dynamics $f$ may be linearized about the reference trajectory $(\bar u_{0:T-1}, \bar x_{0:T})$ as
\begin{equation*}
    \underbrace{f(x_t, u_t) - f(\bar x_t, \bar u_t)}_{\Delta x_{t+1}} \,\simeq\, \nabla_x f(\bar x_t, \bar u_t)^\top \underbrace{(x_t - \bar x_t)}_{\Delta x_t} + \nabla_u f(\bar x_t, \bar u_t)^\top \underbrace{(u_t - \bar u_t)}_{\Delta u_t}\text,
\end{equation*}
which is a valid approximation if the perturbations $\Delta x_{0:T}$ and $\Delta u_{0:T-1}$ are small.
By collecting a dataset of $M \in \mathbb N$ rollouts with slightly perturbed actions, we can thus estimate the Jacobians by solving the (analytically tractable) least-squares problem
\begin{equation}\label{eq:ls}
    \argmin_{[A_t^\top\ B_t^\top] \in \mathbb R^{D\times (D + K)}} \sum_{i = 1}^M  \norm{A_t^\top \Delta x_t^{(i)} + B_t^\top \Delta u_t^{(i)} - \Delta x_{t + 1}^{(i)}}^2\text.
\end{equation}
This technique is illustrated in \cref{fig:jacobian:a} (dashed purple line).
Using these estimates in \cref{alg:pontryagin} yields a model-free method we call \emph{on-trajectory}, as the gradient estimate relies only on data generated based on the current trajectory (see \cref{alg:on-policy} for details).
We see a connection to on-policy methods in closed-loop reinforcement learning, where the policy gradient estimate (or the Q-update) similarly depends only on data generated under the current policy.
Like on-policy methods, on-trajectory methods will benefit greatly from the possibility of parallel environments, which could reduce the effective complexity of the forward pass stage from $M + 1$ rollouts to that of a single rollout.\looseness=-1

\begin{figure}
    \newsavebox{\figbox}
    \sbox\figbox{\includegraphics[width=\linewidth]{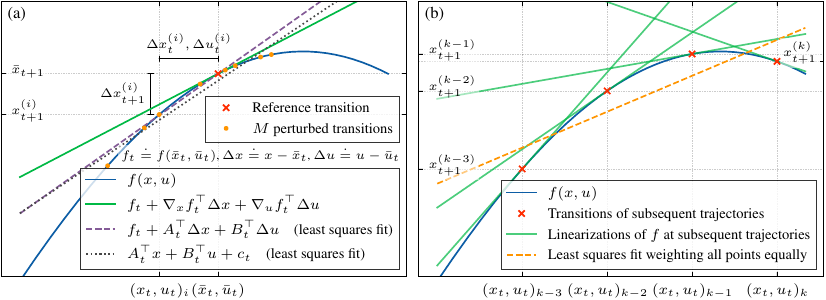}}
    \begin{subfigure}[t][\ht\figbox]{\textwidth}
        \refstepcounter{subfigure}\label{fig:jacobian:a}
        \refstepcounter{subfigure}\label{fig:jacobian:b}
    \end{subfigure}\par\nointerlineskip
    \smash{\usebox{\figbox}}
    \caption{(a) The Jacobians of $f$ (slope of the green linearization) at the reference point $(\bar x_t, \bar u_t)$ can be estimated from the transitions $\{(x_t^{(i)}, u_t^{(i)}, x_{t+1}^{(i)})\}_{i=1}^M$ of $M$ perturbed rollouts. (b) The Jacobians of subsequent trajectories (indexed by $k$) remain close. To estimate the Jacobian at iteration $k$, the most recent iterate ($k - 1$) is more relevant than older iterates.}
    \label{fig:jacobian}
    \vspace{-0.5em}
\end{figure}

\paragraph{Exploiting the Markovian structure.}
Consider a direct linearization of the objective function $J$ about the current trajectory.
Writing the action sequence as a vector $\bar{\bm u} \doteq \operatorname{vec}(\bar u_{0:T-1}) \in \mathbb R^{TK}$, this linearization is given, for $\bm u \in \mathbb R^{TK}$ close to $\bar{\bm u}$, by
\[
    J(\bm u) \simeq J(\bar{\bm u}) + \nabla J(\bar{\bm u})^\top (\bm u - \bar{\bm u}).
\]
We can thus estimate the gradient of the objective function by solving the least squares problem
\[
    \nabla J(\bar{\bm u}) \simeq \argmin_{\bm g \in \mathbb R^{TK}} \sum_{i=1}^M \{J(\bm u_i) - J(\bar{\bm u}) - \bm g^\top (\bm u_i - \bar{\bm u})\}^2\text,
\]
where $\{\bm u_i\}$ are $M \in \mathbb N$ slightly perturbed action sequences.
Due to the dimensionality of $\bar{\bm u}$, this method requires $\mathcal O(TK)$ rollouts to estimate the gradient.
In contrast to this, our approach leverages the Markovian structure of the problem, including the fact that we observe the states $x_{0:T}$ in each rollout.
As the Jacobians are estimated jointly at all time steps, we can expect to get a useful gradient estimate from only $\mathcal O(D^2 + DK)$ rollouts, which significantly reduces the sample complexity if $T$ is large.
This gain in efficiency is demonstrated empirically in \cref{sec:experiments}.

\subsection{Model-free off-trajectory open-loop RL} \label{sec:off-policy}
The on-trajectory algorithm is sample-efficient in the sense that it leverages the problem structure, but a key inefficiency remains: the rollout data sampled at each iteration is discarded after the action sequence is updated.
In this section, we propose an \emph{off-trajectory} method that implicitly uses the data from previous trajectories to construct the Jacobian estimates.
Our approach is based on the following observation.
If the dynamics $f$ are smooth and the step size $\eta$ is small, then the updated trajectory $(u_{0:T-1}^{(k)}, x_{0:T}^{(k)})$ will remain close to the previous iterate $(u_{0:T-1}^{(k - 1)}, x_{0:T}^{(k - 1)})$.
Furthermore, the Jacobians along the updated trajectory will be similar to the previous Jacobians, as illustrated in \cref{fig:jacobian:b}.
Thus, we propose to estimate the Jacobians along the current trajectory from \emph{a single rollout only} by bootstrapping our estimates using the Jacobian estimates from the previous iteration.

Consider again the problem of estimating the Jacobians from multiple perturbed rollouts, illustrated in \cref{fig:jacobian:a}.
Instead of relying on a reference trajectory and \cref{eq:ls}, we can estimate the Jacobians by fitting a linear regression model to the dataset of $M$ perturbed transitions.
Solving
\begin{equation}\label{eq:lslin}
    \argmin_{[A_t^\top\ B_t^\top\ c_t] \in \mathbb R^{D\times (D + K + 1)}} \sum_{i = 1}^M \norm{A_t^\top x_t^{(i)} + B_t^\top u_t^{(i)} + c_t - x_{t + 1}^{(i)}}^2
\end{equation}
yields an approximate linearization $f(x_t, u_t) \simeq A_t^\top x_t + B_t^\top u_t + c_t = F_t z_t$,
with $F_t \doteq [A_t^\top\ B_t^\top\ c_t]$ and $z_t \doteq (x_t, u_t, 1) \in \mathbb R^{D+K+1}$.
This approximation is also shown in \cref{fig:jacobian:a} (dotted gray line).\footnote{If we replace \cref{eq:ls} by \cref{eq:lslin} in \cref{alg:on-policy}, we get a slightly different on-trajectory method with similar performance.}
At iteration $k$, given the estimate $F_t^{(k-1)}$ and a new point $z_t^{(k)} = (x_t^{(k)}, u_t^{(k)}, 1)$ with corresponding target $x_{t+1}^{(k)}$, computing the new estimate $F_t^{(k)}$ is a problem of online linear regression.
We solve this regression problem using an augmented version of the \emph{recursive least squares} (RLS) algorithm \citep[e.g.,][Sec.\ 11.2]{ljung1999system}.
By introducing a prior precision matrix $Q_t^{(0)} \doteq q_0 I$ for each time $t$, where $q_0 > 0$, we compute the update at iteration $k \in \mathbb N$ (see \cref{alg:off-policy}) as
\begin{align}
    Q_t^{(k)} &= \alpha Q_t^{(k-1)} + (1 - \alpha) q_0I + z_t^{(k)} \{z_t^{(k)}\}^\top\label{eq:rls-aug}\\
    F_t^{(k)} &= F_t^{(k-1)} + \{Q_t^{(k)}\}^{-1} z_t^{(k)} \{x_{t+1}^{(k)} - F_t^{(k-1)} z_t^{(k)}\}^\top\nonumber\text.
\end{align}

\paragraph{Forgetting and stability.}
The standard RLS update of the precision matrix corresponds to \cref{eq:rls-aug} with $\alpha = 1$.
In the limit as $q_0 \to 0$, the RLS algorithm is equivalent to the batch processing of \cref{eq:lslin}, which treats all points equally.
However, as illustrated in \cref{fig:jacobian:b}, points from recent trajectories should be given more weight, as transitions that happened many iterations ago will give little information about the Jacobians along the current trajectory.
We can incorporate a \emph{forgetting factor} $\alpha \in (0, 1)$ into the precision update with the effect that past data points are exponentially downweighted:\looseness=-1
\begingroup
\abovedisplayskip=5pt
\belowdisplayskip=5pt
\begin{equation}
    Q_t^{(k)} = \alpha Q_t^{(k-1)} + z_t^{(k)} \{z_t^{(k)}\}^\top\quad\leadsto\quad Q_t^{(k)} = \alpha^k q_0I + \sum_{i=1}^k \alpha^{k - i} z_t^{(i)} \{z_t^{(i)}\}^\top.\label{eq:rls}
\end{equation}
\endgroup
This forgetting factor introduces a new problem: instability.
If subsequent trajectories lie close to each other, then the sum of outer products may become singular (e.g., if all $z_t^{(i)}$ are identical, then the sum has rank $1$).
As the prior $q_0I$ is downweighted, at some point inverting $Q$ may become numerically unstable.
Our modification in \cref{eq:rls-aug} adds $(1 - \alpha) q_0I$ in each update, which has the effect of removing the $\alpha^k$ coefficient in front of $q_0I$ in \cref{eq:rls}.
If the optimization procedure converges, then eventually subsequent trajectories will indeed lie close together.
Although \cref{eq:rls-aug} prevents issues with instability, the quality of the Jacobian estimates will still degrade, as this estimation inherently requires perturbations (see \cref{sec:on-policy}).
In \cref{alg:off-policy}, we thus slightly perturb the actions used in each rollout to get more diverse data.\looseness=-1

\section{Experiments}\label{sec:experiments}
\subsection{Inverted pendulum swing-up}\label{sec:pendulum}
We empirically evaluate our algorithms on the inverted pendulum swing-up task shown in \cref{fig:pendulum}.
As a performance criterion we define $J_\mathrm{max} \doteq \max_{k \in [N]_0} J(u_{0:T-1}^{(k)})$ as the return achieved by the best action sequence over a complete learning process of $N$ optimization steps.
The task is considered \emph{solved} if $J_\mathrm{max}$ exceeds a certain threshold.
A detailed description of the task is given in \cref{app:pendulum}.
We repeat our experiments with $100$ random seeds and show $95\%$ bootstrap confidence intervals in all plots.\looseness=-1
\begin{figure}
  \centering
  \begin{minipage}[t]{.46\textwidth}
    \vspace{0pt}
      \includegraphics[width=\textwidth]{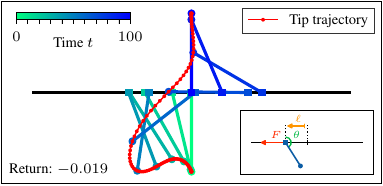}
      \caption{The inverted pendulum swing-up task. The goal is to control the force $F$ such that the tip of the pendulum swings up above the base. The shown solution was found by the on-trajectory method of \cref{sec:on-policy}.}
      \label{fig:pendulum}
  \end{minipage}\hfill
  \begin{minipage}[t]{0.51\textwidth}
    \vspace{0pt}
    \includegraphics[width=\textwidth]{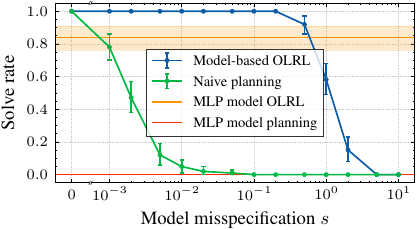}
    \vspace{-12pt}
    \caption{The model-based open-loop RL algorithm can solve the pendulum problem reliably even with a considerable model error.}
    \label{fig:robustness}
  \end{minipage}
  \vspace{-1.1em}
\end{figure}
% The state at time $t$ is $x_t = (\ell, \dot \ell, \theta, \dot \theta)_t \in \mathbb R^4$, where $\ell$ is the position of the cart on the bar and $\theta$ is the pendulum angle.
% The action $u_t$ is the horizontal force $F$ applied to the cart at time $t$.
% Episodes are of length $T = 100$, the running reward $r(x, u) \propto -u^2$ penalizes large forces, and the terminal reward $r_T(x) = -\norm{x}_1$ defines the goal state to be at rest in the upright position.

% \paragraph{Performance metrics.}
% We monitor several performance criteria, all of which are based on the return $J$.
% First, we define $J_\mathrm{max} \doteq \max_{k \in [N]_0} J(u_{0:T-1}^{(k)})$ as the return achieved by the best action sequence over a complete learning process of $N$ optimization steps.
% Based on this quantity, we say that the swing-up task is \emph{solved} if $J_\mathrm{max} > -0.03$.
% This threshold was determined empirically. 
% If the algorithm or the model is randomized, then $[J_\mathrm{max} > -0.03]$ is a Bernoulli random variable whose mean, which we call the \emph{solve rate}, depends on the quality of the learning algorithm.

\paragraph{Robustness: model-based open-loop RL.}
In \cref{thm:main}, we proved that our model-based open-loop RL method (\cref{alg:model-based}) can accommodate some model error and still converge to a local maximum of the true objective.
% In \cref{fig:robustness}, we empirically analyze this robustness property on the pendulum system.
To test the robustness of our algorithm against model misspecification, we use a pendulum system with inaccurate parameters as the model $\tilde f$.
Concretely, if $m_i$ is the $i$\textsuperscript{th} parameter of the true system (cf.\ \cref{app:pendulum}), we sample the corresponding model parameter $\tilde m_i$ from a log-normal distribution centered at $m_i$, such that $\tilde m_i = \xi m_i$, with $\ln \xi\sim \mathcal N(0, s^2)$.
The severity of the model error is then controlled by the scale parameter $s$.
In \cref{fig:robustness}, we compare the performance of our method with the planning procedure described in \cref{sec:model-based}, in which the forward pass is performed through the model $\tilde f$ instead of the real system $f$.
Whereas the planning method only solves the pendulum reliably with the true system as the model ($s = 0$), the open-loop RL method can accommodate a considerable model misspecification.

In a second experiment, we represent the model $\tilde f$ by a small multi-layer perceptron (MLP).
The model is learned from $1000$ rollouts, with the action sequences sampled from a pink noise distribution, as suggested by \citet{eberhard-2023-pink}.
\Cref{fig:robustness} compares the performance achieved with this model by our algorithm and by the planning method.
As the MLP model represents a considerable misspecification of the true dynamics, only the open-loop RL method manages to solve the pendulum task.\looseness=-1

\paragraph{Structure: on-trajectory open-loop RL.}
Our model-free on-trajectory method (\cref{alg:on-policy}) uses rollouts to directly estimate the Jacobians needed to update the action sequence.
% In \cref{sec:on-policy}, we proposed a model-free approach (\cref{alg:on-policy}) that uses rollouts to directly estimate the Jacobians needed to update the action sequence.
It is clear from \cref{eq:ls} that more rollouts (i.e., larger $M$) will give more accurate Jacobian estimates, and therefore increase the quality of the gradient approximation.
In \cref{fig:rollouts}, we analyze the sample efficiency of the this algorithm by comparing the performance achieved at different values of $M$, where the number $N$ of optimization steps remains fixed.
We compare our method to the finite-difference approach described at the end of \cref{sec:on-policy} and to the gradient-free cross-entropy method \citep[CEM;][]{rubinstein1999cross}.
Both these methods also update the action sequence on the basis of $M$ perturbed rollouts in the environment.
As in our method, the $M$ action sequences are perturbed using Gaussian white noise with noise scale $\sigma$.
We describe both baselines in detail in \cref{app:baselines}.
The \emph{oracle} performance shown in \cref{fig:rollouts} corresponds to \cref{alg:pontryagin} with the true gradient, i.e., $A_t = \nabla_x f_t$ and $B_t = \nabla_u f_t$.

\Cref{fig:rollouts} shows that the performance of both the finite-difference method and CEM heavily depends on the choice of the noise scale $\sigma$, whereas our method performs identically for all three values of $\sigma$.
Even for tuned values of $\sigma$, the finite-difference method and CEM still need approximately twice as many rollouts per iteration as the open-loop RL method to reliably swing up the pendulum.
At $10$ rollouts per iteration, our method matches the oracle's performance, while both baselines are below the \emph{solved} threshold.
This empirically confirms our theoretical claims at the end of \cref{sec:on-policy}, where we argue that exploiting the Markovian structure of the problem leads to increased sample efficiency.\looseness=-1

\paragraph{Efficiency: off-trajectory open-loop RL.}
Finally, we turn to the method proposed in \cref{sec:off-policy} (\cref{alg:off-policy}), which promises increased sample efficiency by estimating the Jacobians in an off-trajectory fashion.
The performance of this algorithm is shown in \cref{fig:learning}, where the learning curves of all our methods as well as the two baselines and the oracle are plotted.
For the on-trajectory methods compared in \cref{fig:rollouts}, we chose for each the minimum number of rollouts $M$ such that, under the best choice for $\sigma$, the method would reliably solve the swing-up task.
The hyperparameters for all methods are summarized in \cref{app:hyper}.
It can be seen that the off-trajectory method, which only requires one rollout per iteration, converges much faster than the on-trajectory open-loop RL method.

\begin{figure}
  \centering
  \begin{minipage}[t]{.49\textwidth}
    \vspace{0pt}
      \includegraphics[width=\textwidth]{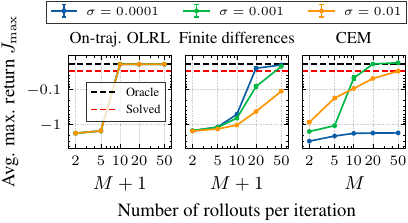}
      \caption{The on-trajectory open-loop RL method is more sample-efficient than the finite-difference and cross-entropy methods. It is also much less sensitive to the noise scale $\sigma$.}
      \label{fig:rollouts}
  \end{minipage}\hfill
  \begin{minipage}[t]{.49\textwidth}
    \vspace{0pt}
    \includegraphics[width=\textwidth]{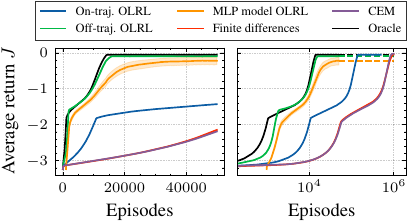}
    \caption{Learning curves on the pendulum task. On the right, we show a longer time period in log scale. The off-trajectory open-loop RL method converges almost as fast as the oracle method.}
    \label{fig:learning}
  \end{minipage}
  \vspace{-1em}
\end{figure}
  
\subsection{MuJoCo}\label{sec:mjc}
\begin{wrapfigure}{r}{0.6\textwidth}
  \vspace{-3em}
  \centering
  \includegraphics[width=\linewidth]{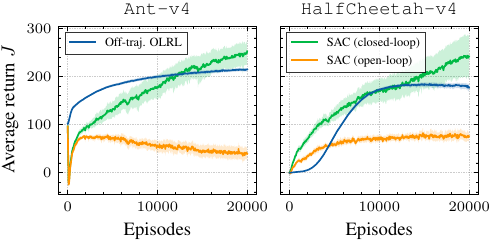}
  \caption{Learning curves of our off-trajectory open-loop RL method and soft actor-critic (SAC) for two MuJoCo tasks. All experiments were repeated with $20$ random seeds, and we show $95\%$-bootstrap confidence intervals for the average return. The horizon is fixed to $T = 100$.}
  \label{fig:mjc}
  \vspace{-1.2em}
\end{wrapfigure}
While the inverted pendulum is illustrative for analyzing our algorithms empirically, it is a relatively simple task with smooth, low-dimensional, and deterministic dynamics.
In this section, we test our method in two considerably more challenging environments: the \texttt{Ant-v4} and \texttt{HalfCheetah-v4} tasks provided by the OpenAI Gym library \citep{brockman-2016-gym,towers-2023-gymnaisum}, implemented in MuJoCo \citep{todorov2012mujoco}.
These environments are high-dimensional, they exhibit non-smooth contact dynamics, and the initial state is randomly sampled at the beginning of each episode.

We tackle these two tasks with our model-free off-trajectory method (\cref{alg:off-policy}).
The results are shown in \cref{fig:mjc}, where we compare to the closed-loop RL baseline soft actor-critic \citep[SAC;][]{haarnoja-2018a-soft}.
It can be seen that the open-loop RL method performs comparably to SAC, even though SAC learns a closed-loop policy that is capable of adapting its behavior to the initial condition.\footnote{In this comparison, our method is further disadvantaged by the piecewise constant ``health'' terms in the reward function of \texttt{Ant-v4}. Our method, exclusively relying on the gradient of the reward function, ignores these.\looseness=-1}
In the figure, we also analyze the open-loop performance achieved by SAC.
Whereas the closed-loop performance is the return obtained in a rollout where the actions are taken according to the mean of the Gaussian policy, the open-loop return is achieved by blindly executing exactly the same actions in a new episode.
The discrepancy in performance is thus completely due to the stochasticity in the initial state.
In \cref{app:mjc-long}, we show that our method also works with a longer horizon $T$.\looseness=-1

The results demonstrate that the open-loop RL algorithm is robust to a certain level of stochasticity in the initial state of stable dynamical systems.
Additionally, while our convergence analysis depends on the assumption of smooth dynamics, these experiments empirically demonstrate that the algorithms are also able to tackle non-smooth contact dynamics.
Finally, we see that the high dimensionality of the MuJoCo systems is handled without complications.
While soft actor-critic is an elegant and powerful algorithm, the combination with deep function approximation can make efficient learning more difficult.
Our methods are considerably simpler and, because they are based on Pontryagin's principle rather than dynamic programming, they evade the curse of dimensionality by design, and thus do not require any function approximation.

\section{Discussion}\label{sec:discussion}
This paper makes an important first step towards understanding how principles from open-loop optimal control can be combined with ideas from reinforcement learning while preserving convergence guarantees.
We propose three algorithms that address this \emph{open-loop RL} problem, from robust trajectory optimization with an approximate model to sample-efficient learning under fully unknown dynamics.
This work focuses on reinforcement learning in continuous state and action spaces, a class of problems known to be challenging \citep{recht2019}.
Although this setting allows us to leverage continuous optimization techniques, we expect that most ideas will transfer to the discrete setting, and we would be interested to see further research on this topic.

It is interesting to note that there are many apparent parallels between our open-loop RL algorithms and their closed-loop counterparts.
The distinction between model-based and model-free methods is similar to that in closed-loop RL.
Likewise, the on-trajectory and off-trajectory methods we present show a tradeoff between sample efficiency and stability that is reminiscent of the tradeoffs between on-policy and off-policy methods in closed-loop RL.
The question of exploration, which is central to reinforcement learning, also arises in our case.
We do not address this complex problem thoroughly here but instead rely on additive Gaussian noise to sample diverse trajectories.

\paragraph{Limitations of open-loop RL.}
Another important open question is how open-loop methods fit into the reinforcement learning landscape.
An inherent limitation of these methods is that an open-loop controller can, by definition, not react to unexpected changes in the system's state, be it due to random disturbances or an adversary.
An open-loop controller cannot balance an inverted pendulum in its unstable position\footnote{Except with a clever trick called \emph{vibrational control} \citep{meerkov-1980-principle}.}, track a reference trajectory in noisy conditions, or play Go, where reactions to the opponent's moves are constantly required.
In these situations, open-loop RL is not viable or only effective over a very short horizon $T$.
However, if the disturbances are small, and the system is not sensitive to small changes in state or action (roughly speaking, if the system is stable and non-chaotic), then a reaction is not necessary, and open-loop RL works even for long horizons $T$ (as we highlight in our MuJoCo experiments, cf.\ \cref{app:mjc-long}).
Open-loop control can be viewed as a special case of closed-loop control, and therefore it is clear that closed-loop control is much more powerful.
Our algorithms provide a first solution to the open-loop RL problem and are not intended to replace any of the existing closed-loop RL algorithms.
In control engineering, it is common to combine feedback and feedforward techniques.
In many situations, it can be shown that such a combination will significantly outperform a solution based on feedback alone \citep[e.g.,][Sec.\ 12.4]{aastrom2021feedback}.
We believe that ultimately a combination of open-loop and closed-loop techniques will also be fruitful in reinforcement learning and think that this is an important direction for future research.\looseness=-1
% We expand on this discussion in \cref{app:limitations}.

\acks{
We thank the International Max Planck Research School for Intelligent Systems (IMPRS-IS) for their support.
C.\ Vernade is funded by the German Research Foundation (DFG) under both the project 468806714 of the Emmy Noether Programme and under Germany's Excellence Strategy -- EXC number 2064/1 -- Project number 390727645.
M.\ Muehlebach is funded by the German Research Foundation (DFG) under the project 456587626 of the Emmy Noether Programme.
}

% \subsubsection*{Reproducibility statement}
% We include all code required for reproduction of our results in the supplementary material.
% Detailed descriptions of all algorithms (including baselines) are provided in \cref{app:algorithms,app:baselines}, and all hyperparameters are listed in \cref{app:hyper}.
% All empirical results include bootstrap confidence intervals based on multiple random seeds ($100$ for the inverted pendulum experiments and $20$ for the MuJoCo experiments).
% Our theoretical result (\cref{thm:main}) is proved in \cref{app:proof}, and all assumptions are explained in detail in \cref{sec:theorem}.

\bibliography{bib}
\appendix
\counterwithin{algocf}{section}

\iftoggle{noappendix}{
\refstepcounter{section}\label{app:related-work}
}{
\clearpage
\section{Related work}\label{app:related-work}
Historically, control theory has dealt with both closed-loop and open-loop control, and there is a broad consensus in the control community that both are important \citep{aastrom2021feedback,skogestad2005multivariable,astrom1995pid,betts2010practical,verscheure-2009-time,horn-2013-numerical}.
A simple application of open-loop methods is the control of electric stoves with \emph{simmerstats}, which regulate the temperature by periodically switching the power on and off.
Open-loop control is also applied to much more challenging problems, such as the regulation of drinking-water treatment plants \citep{gamiz-2020-feed} or plasmas in a tokamak \citep{mattei-2006-open}.
These problems have also been of interest to the reinforcement learning community \citep{janjua2024gvfs,degrave2022magnetic}.\looseness=-1

The numerical solution of trajectory optimization problems has also been studied in machine learning \citep{schaal2010,howe2022myriad}.
As in our approach, an important aspect of these methods is to exploit the Markovian structure of the dynamics to reduce computation \citep{Carraro2015,diehl2007}.
However, in contrast to the RL setting that we consider, existing methods address situations where the dynamics are known.
Another set of related methods is known as \emph{iterative learning control} \citep{moore1993,ma2022,ma2023}, which is a control-theoretic framework that iteratively improves the execution of a task by optimizing over feedforward trajectories.
However, these methods are often formulated for \textit{trajectory tracking} tasks, while we consider a more general class of reinforcement learning problems.
\Citet{chen-2019-hardware} explore an idea similar to that in our \cref{alg:model-based}; their model-based control algorithm combines a rollout in a real system with an inaccurate model to construct an iterative LQR feedback controller.
A combination of open-loop and closed-loop methods in the context of reinforcement learning is explored by \citet{hansen1996reinforcement}.\looseness=-1

\citet{raffin2023simple} have recently proposed to use open-loop control as a baseline to compare to more complex deep reinforcement learning methods.
They argue, as do we, that deep RL methods have become very complex, and that open-loop solutions may be favored for certain tasks due to their simplicity.
Their approach is to combine a gradient-free optimization method (CMA-ES, very similar to our baseline CEM), with prior knowledge about the problem (instead of letting the $u_t$ be completely free, they optimize the parameters of nonlinear oscillators).
The authors express their surprise about the good performance of the open-loop method compared to state-of-the-art deep RL algorithms on certain MuJoCo tasks, similar to the ones we consider in \cref{sec:mjc} \citep[cf.][p.~8]{raffin2023simple}.\looseness=-1

Recently, deep neural networks have been used to learn representations of complex dynamical systems \citep{fragkiadaki2015recurrent} and Pontryagin's principle was leveraged in the optimization of control tasks based on such models \citep{jin2020pontryagin,bottcher2022ai}.
However, these methods only consider the setting of closed-loop control.
The combination of an exact forward pass with an approximate backward pass, which our methods are based on, has also been explored in different settings in the deep learning literature, such as spiking \citep{lee-2016-training} or physical \citep{Wright2022} neural networks, or networks that include nondifferentiable procedures, for example used for rendering \citep{Niemeyer2020CVPR} or combinatorial optimization \citep{pogancic2020}.
The analysis of \cref{app:proof} that we developed for our convergence result (\cref{thm:main}) could also be relevant for these applications, as the fundamental structure (backpropagation with an accurate forward pass and an inaccurate backward pass) is identical.

}

\iftoggle{noappendix}{
\refstepcounter{section}\label{app:algorithms}
\refstepcounter{algocf}\label{alg:model-based}
\refstepcounter{algocf}\label{alg:on-policy}
\refstepcounter{algocf}\label{alg:off-policy}
}{
\section{Algorithms}\label{app:algorithms}
In this section, we provide detailed descriptions of the three open-loop RL algorithms presented in the main text.
The model-based algorithm of \cref{sec:model-based} is listed in \cref{alg:model-based}, the model-free on-trajectory method of \cref{sec:on-policy} is listed in \cref{alg:on-policy}, and the off-trajectory method of \cref{sec:off-policy} is listed in \cref{alg:off-policy}.
The hyperparameters we use in these algorithms are discussed in \cref{app:hyper}.
\begin{algorithm2e}
  \caption{Model-based open-loop RL}\label{alg:model-based}
  \KwIn{Differentiable model $\tilde f: \mathcal X \times \mathcal U \to \mathcal X$, optimization steps $N \in \mathbb N$, step size $\eta > 0$}
  Initialize $u_{0:T-1}$ (initial action sequence)\\
  \For{$k = 1, 2, \dots, N$}{
      \tcp{Forward pass}
      $x_{0:T} \gets \operatorname{rollout}(u_{0:T-1})$\\
      \BlankLine
      \tcp{Backward pass}
      $\tilde\lambda_T \gets \nabla r_T(x_T)$\\
      \For{$t = T-1, T-2, \dots, 0$}{
          $\tilde\lambda_t \gets \nabla_x r(x_t, u_t) + \nabla_x\tilde f(x_t, u_t)\tilde\lambda_{t+1}$\\
          $g_t \gets \nabla_u r(x_t, u_t) + \nabla_u\tilde f(x_t, u_t)\tilde\lambda_{t+1}$\\
          $u_t \gets u_t + \eta g_t$\hspace{1em}\tcp{Gradient ascent}
      }
  }
\end{algorithm2e}
\begin{algorithm2e}
    \caption{Model-free on-trajectory open-loop RL}\label{alg:on-policy}
    \KwIn{Number of rollouts $M \in \mathbb N$, noise scale $\sigma > 0$, optimization steps $N \in \mathbb N$,\newline step size $\eta > 0$}
    Initialize $\bar u_{0:T-1}$ (initial action sequence)\\
    \For{$k = 1, 2, \dots, N$}{
        \tcp{Forward passes}
        $\bar x_{0:T} \gets \operatorname{rollout}(\bar u_{0:T-1})$\\
        \For{$i = 1, 2, \dots, M$}{
            $u_{0:T-1}^{(i)} \sim \mathcal N(\bar u_{0:T-1}, \sigma I)$\\
            $x_{0:T}^{(i)} \gets \operatorname{rollout}(u_{0:T-1}^{(i)})$\\
            $\Delta u_{0:T-1}^{(i)} \gets u_{0:T-1}^{(i)} - \bar u_{0:T-1}$\\
            $\Delta x_{0:T}^{(i)} \gets x_{0:T}^{(i)} - \bar x_{0:T}$
        }
        \BlankLine
        \tcp{Backward pass}
        $\tilde\lambda_T \gets \nabla r_T(\bar x_T)$\\
        \For{$t = T-1, T-2, \dots, 0$}{
            \tcp{Jacobian estimation}
            $A_t, B_t \gets \argmin_{A_t \in \mathbb R^{D\times D}, B_t \in \mathbb R^{K\times D}}
            \sum_{i=1}^M \norm{A_t^\top \Delta x_t^{(i)} + B_t^\top \Delta u_t^{(i)} - \Delta x_{t+1}^{(i)}}^2$\\
            \BlankLine
            \tcp{Pontryagin update}
            $\tilde\lambda_t \gets \nabla_x r(\bar x_t, \bar u_t) + A_t\tilde\lambda_{t+1}$\\
            $g_t \gets \nabla_u r(\bar x_t, \bar u_t) + B_t\tilde\lambda_{t+1}$\\
            $\bar u_t \gets \bar u_t + \eta g_t$\hspace{1em}\tcp{Gradient ascent}
        }
    }
\end{algorithm2e}
\begin{algorithm2e}
    \caption{Model-free off-trajectory open-loop RL}\label{alg:off-policy}
    \KwIn{Forgetting factor $\alpha \in [0, 1]$, noise scale $\sigma > 0$, initial precision $q_0 > 0$,\newline optimization steps $N \in \mathbb N$, step size $\eta > 0$}
    Initialize $\bar u_{0:T-1}$ (initial action sequence)\\
    Initialize $F_t \in \mathbb R^{D \times (D + K + 1)}, \forall t \in [T - 1]_0$\\
    $Q_t \gets q_0 I \in \mathbb R^{(D + K + 1) \times (D + K + 1)}, \forall t \in [T-1]_0$\\
    \For{$k = 1, 2, \dots, N$}{
        \tcp{Forward pass}
        $u_{0:T-1} \sim \mathcal N(\bar u_{0:T-1}, \sigma I)$\\
        $x_{0:T} \gets \operatorname{rollout}(u_{0:T-1})$\\
        \BlankLine
        \tcp{Backward pass}
        $\tilde\lambda_T \gets \nabla r_T(x_T)$\\
        \For{$t = T-1, T-2, \dots, 0$}{
            \tcp{Jacobian estimation}
            $z_t \gets [x_t^\top\ u_t^\top\ 1]^\top$\\
            $Q_t \gets \alpha Q_t + (1 - \alpha) q_0 I + z_t z_t^\top$\\
            $F_t \gets F_t + Q_t^{-1} z_t (x_{t+1} - F_t z_t)^\top$\\
            $[A_t^\top\ B_t^\top\ c_t] \gets F_t$\\
            \BlankLine
            \tcp{Pontryagin update}
            $\tilde\lambda_t \gets \nabla_x r(x_t, u_t) + A_t\tilde\lambda_{t+1}$\\
            $g_t \gets \nabla_u r(x_t, u_t) + B_t\tilde\lambda_{t+1}$\\
            $\bar u_t \gets \bar u_t + \eta g_t$\hspace{1em}\tcp{Gradient ascent}
        }
    }
\end{algorithm2e}
}

\iftoggle{noappendix}{
\refstepcounter{section}\label{app:lagrange}
}{
\section{Derivation of Pontryagin's principle}\label{app:lagrange}
In this section, we will derive Pontryagin's principle, equations \cref{eq:gradient,eq:costate1,eq:costate2}, using the method of Lagrange multipliers.
In the following we view the objective $J$ as a function of states and actions, that is
\begin{equation*}
    J(x_{0:T}, u_{0:T-1}) \doteq \sum_{t=0}^{T-1} r(x_t, u_t) + r_T(x_T).
\end{equation*}
We maximize $J$ with respect to $x_{0:T}$ and $u_{0:T-1}$ subject to the constraint that $x_{t+1} = f(x_t, u_t)$ for all $t = [T-1]_0$.
The corresponding Lagrangian is
\begin{equation*}
    L(x_{0:T}, u_{0:T-1}, \lambda_{1:T}) \doteq \sum_{t=0}^{T-1}\{r(x_t, u_t) + \lambda_{t+1}^\top (f(x_t, u_t) - x_{t+1})\} + r_T(x_T)\text,
\end{equation*}
where the constraints are included through the multipliers $\lambda_{1:T}$.
The costate equations are then obtained by setting the partial derivatives of the Lagrangian with respect to $x_{0:T}$ to zero:
\begin{alignat*}{3}
    &&\nabla_{x_t} L &\,=\,\nabla_x r(x_t, u_t) + \nabla_x f(x_t, u_t) \lambda_{t + 1} - \lambda_t \doteq 0\\
    &\implies\quad& \lambda_t &\,=\,\nabla_x r(x_t, u_t) + \nabla_x f(x_t, u_t) \lambda_{t + 1}\\
    &&\nabla_{x_T} L &\,=\,\nabla r_T(x_T) - \lambda_T \doteq 0\\
    &\implies\quad& \lambda_T &\,=\,\nabla r_T(x_T).
\end{alignat*}
Setting the partial derivatives of the Lagrangian with respect to $\lambda_{1:T}$ to zero yields the dynamics equations, and the partial derivatives of the Lagrangian with respect to $u_{0:T-1}$ are
\begin{equation*}
    \nabla_{u_t} L = \nabla_u r(x_t, u_t) + \nabla_u f(x_t, u_t) \lambda_{t + 1}\text,
\end{equation*}
which is the same expression for the gradient of the objective as in \cref{eq:gradient}.
}

\iftoggle{noappendix}{
\refstepcounter{section}\label{app:backprop}
}{
\section{Pontryagin's principle from backpropagation}\label{app:backprop}
In \cref{sec:model-based}, we mention that an application of the backpropagation algorithm (i.e., a repeated application of the chain rule) to the objective $J$ leads naturally to Pontryagin's principle.
We have
\begin{align*}
    J(u_{0:T-1}) &= \sum_{t=0}^{T-1} r(x_t, u_t) + r_T(x_T)\\
    &= \begin{multlined}[t]
        r(x_0, u_0) + r\{\underbrace{f(x_0, u_0)}_{x_1}, u_1\} + r\{\underbrace{f(f(x_0, u_0), u_1)}_{x_2}, u_2\}+ \cdots\\
        + r\{\underbrace{f(f(\cdots f(f(x_0, u_0), u_1)\cdots), u_{T-2})}_{x_{T-1}}, u_{T-1}\}\\
        + r_T\{\underbrace{f(f(\cdots f(f(x_0, u_0), u_1)\cdots), u_{T-1})}_{x_{T}}\}\text.
    \end{multlined}
\end{align*}
The chain rule states that for $g: \mathbb R^n \to \mathbb R^k$, $h: \mathbb R^k \to \mathbb R^m$ and $x \in \mathbb R^n$,
\[
    \nabla (h \circ g)(x) = \nabla g(x)\,\nabla h\{g(x)\}\text,
\]
where $\nabla g: \mathbb R^n \to \mathbb R^{n \times k}$, $\nabla h: \mathbb R^k \to \mathbb R^{k \times m}$ and $\nabla (h \circ g): \mathbb R^n \to \mathbb R^{n \times m}$.
From this, we can compute the gradient of the objective function with respect to the action $u_t$ at time $t \in [T-1]_0$ as
\begin{align*}
\nabla_{u_t} J(u_{0:T-1}) &= \begin{aligned}[t]
    &\nabla_u r(x_t, u_t) + \nabla_u f(x_t, u_t)\textcolor{blue}{\nabla_x r(x_{t+1}, u_{t+1})}\\
    &\quad+ \nabla_u f(x_t, u_t)\textcolor{blue}{\nabla_x f(x_{t+1}, u_{t+1})\nabla_x r(x_{t+2}, u_{t+2})}\\
    &\quad+ \cdots\\
    &\quad+ \nabla_u f(x_t, u_t)\textcolor{blue}{\nabla_x f(x_{t+1}, u_{t+1})\cdots \nabla_x f(x_{T-2}, u_{T-2}) \nabla_x r(x_{T-1}, u_{T-1})}\\
    &\quad+ \nabla_u f(x_t, u_t) \textcolor{blue}{\nabla_x f(x_{t+1}, u_{t+1}) \cdots \nabla_x f(x_{T-1}, u_{T-1}) \nabla r_T(x_{T}) }
\end{aligned}\\
&= \nabla_u r(x_t, u_t) + \nabla_u f(x_t, u_t) \lambda_{t+1}\text,
\end{align*}
where we have introduced the shorthand $\lambda_{t+1}$ for the blue part.
This is the same expression for the gradient as in \cref{eq:gradient}, and it can easily be seen that this definition of $\lambda_{t}$ satisfies the costate equations \cref{eq:costate1,eq:costate2}.
}

\iftoggle{noappendix}{
\refstepcounter{section}\label{app:proof}
}{
\section{Proof of \texorpdfstring{\cref{thm:main}}{Theorem~\ref{thm:main}}}\label{app:proof}
In this section, we prove \cref{thm:main}, our convergence result of \cref{alg:pontryagin}.
The main part of the proof is contained in the proof of \cref{thm:big}, which provides a lower bound for the inner product between the approximate and true gradients as well as an upper bound for the norm of the approximate gradients.
Intuitively, this theorem turns \cref{ass:error}, which is a statement about the error of the approximate Jacobians, into a statement about the error of the approximate gradient.
We then show that \cref{thm:main} follows by making use of the $L$-smoothness (\cref{ass:smooth}) of the objective function.
This latter part is a standard result in the analysis of stochastic gradient methods \citep[e.g.,][]{bottou2018optimization}.

Before coming to the main result, we introduce the following shorthand notation.
Given a fixed trajectory $(u_{0:T-1}, x_{0:T})$, we define
\begin{equation*}
\nabla J_t \doteq \nabla_u f_t \lambda_{t+1}\text,\quad
\varepsilon_t \doteq A_t - \nabla_x f_t\text,
\quad \varepsilon_t' \doteq B_t - \nabla_u f_t\quad\text{and}\quad
\delta_{t+1} \doteq \tilde\lambda_{t+1} - \lambda_{t+1}
\end{equation*}
for all times $t \in [T-1]_0$.
By \cref{eq:gradient,ass:reward}, the first quantity defines the true gradient and the approximate gradient is given by $g_t = B_t \tilde\lambda_{t+1}$.
We also state two small lemmas, which we will use routinely in the following proof.

\begin{lemma}\label{lem:cauchy}
    Let $x, y \in \mathbb R^n$ for $n \in \mathbb N$ and $\alpha \in \mathbb R$ such that $\norm{x} \leq \alpha \norm{y}$. Then, $|x^\top y| \leq \alpha \norm{y}^2$.
\end{lemma}
\begin{proof}
    \vspace{-1em}
    $\norm{x} \leq \alpha \norm{y} \implies \norm{x} \norm{y} \leq \alpha \norm{y}^2 \implies |x^\top y| \leq \alpha \norm{y}^2$ (by Cauchy-Schwarz).
\end{proof}

\begin{lemma}\label{lem:sigma}
    Let $A, B \in \mathbb R^{m\times n}$ for some $m, n \in \mathbb N$ and $x, y \in \mathbb R^n$ such that $\smax(A) \norm{x} \leq \smin(B) \norm{y}$. Then, $\norm{Ax} \leq \norm{By}.$
\end{lemma}
\begin{proof}
    This is a simple corollary of the Courant-Fischer (min-max) theorem. The min-max theorem states that, for a symmetric matrix $C \in \mathbb R^{n\times n}$, the minimum and maximum eigenvalues $\underaccent{\bar}{\lambda}(C)$ and $\bar\lambda(C)$ are characterized in the following way:
    \[
        \underaccent{\bar}{\lambda}(C) = \min_{\substack{z\in\mathbb R^n\\\norm{z}=1}} z^\top C z \quad \text{and} \quad \bar{\lambda}(C) = \max_{\substack{z\in\mathbb R^n\\\norm{z}=1}} z^\top C z.
    \]
    This can be extended to a characterization of the singular values $\smax(A)$ and $\smin(B)$ by relating them to the eigenvalues of $A^\top A$ and $B^\top B$, respectively:
    \begin{align*}
        \smax(A) &= \sqrt{\bar{\lambda}(A^\top A)} = \max_{\substack{z\in\mathbb R^n\\\norm{z}=1}} \sqrt{z^\top A^\top A z} = \max_{\substack{z\in\mathbb R^n\\\norm{z}=1}} \norm{Az} \geq \frac{1}{\norm{x}}\norm{Ax}\text,\\
        \smin(B) &= \sqrt{\underaccent{\bar}{\lambda}(B^\top B)} = \min_{\substack{z\in\mathbb R^n\\\norm{z}=1}} \sqrt{z^\top B^\top B z} = \min_{\substack{z\in\mathbb R^n\\\norm{z}=1}} \norm{Bz} \leq \frac{1}{\norm{y}}\norm{By}.
    \end{align*}
    Combining these inequalities, we get:
    \begingroup
    \belowdisplayskip=0pt
    \[
        \norm{Ax} \leq \smax(A) \norm{x} \leq \smin(B) \norm{y} \leq \norm{By}.
    \]
    \endgroup
\end{proof}

\begin{theorem}\label{thm:big}
Suppose Assumptions \ref{ass:reward} and \ref{ass:error} hold with $\gamma$ and $\zeta$ and define $\mu \doteq 1 - \gamma - \zeta - \gamma\zeta$ and $\nu \doteq 1 + \gamma + \zeta + \gamma\zeta$. Then,
\[
    g_t^\top \nabla J_t \geq \mu\norm{\nabla J_t}^2 \quad \text{and} \quad \norm{g_t} \leq \nu\norm{\nabla J_t},
\]
for all $t \in [T-1]_0$.
\end{theorem}
\begin{proof}
    \vspace{-1em}
    Let $t \in [T-1]_0$ be fixed.
    Decomposing the left-hand side of the first inequality, we get
    \begin{align*}
        g_t^\top \nabla J_t &= \tilde\lambda_{t+1}^\top B_t^\top \nabla_u f_t \lambda_{t+1}\\
        &= (\lambda_{t+1} + \delta_{t+1})^\top (\nabla_u f_t + \varepsilon_t')^\top \nabla_u f_t \lambda_{t+1}\\
        &= \norm{\nabla J_t}^2 + \underbrace{\lambda_{t+1}^\top \varepsilon_t'^\top \nabla_u f_t \lambda_{t+1}}_{a} + \underbrace{\delta_{t+1}^\top \nabla_u f_t^\top \nabla_u f_t \lambda_{t+1}}_{b} + \underbrace{\delta_{t+1}^\top \varepsilon_t'^\top \nabla_u f_t \lambda_{t+1}}_{c}\\
        &\geq \norm{\nabla J_t}^2 - |a| - |b| - |c|.
    \end{align*}
    We will now show that
    \[
        |a| \leq \zeta \norm{\nabla J_t}^2 \quad \text{and} \quad |b| \leq \gamma \norm{\nabla J_t}^2 \quad \text{and} \quad |c| \leq \gamma\zeta \norm{\nabla J_t}^2\text,
    \]
    which, when taken together, will give us
    \[
        g_t^\top \nabla J_t \geq (1 - \gamma - \zeta - \gamma\zeta)\norm{\nabla J_t}^2 = \mu\norm{\nabla J_t}^2.
    \]
    We first derive the bound on $|a|$:
    \begin{alignat}{4}
        &&&&\smax(\varepsilon_{t}') &\,\leq\,\zeta \smin(\nabla_u f_{t}) &&\qquad \text{(\cref{ass:uerror})}\nonumber\\
        &&\implies\quad &&\norm{\varepsilon_t'\lambda_{t+1}} &\,\leq\,\zeta\norm{\nabla_u f_t\lambda_{t+1}} &&\qquad \text{(\cref{lem:sigma})}\label{eq:zeta}\\
        &&\implies\quad &&|\underbrace{\lambda_{t+1}^\top\varepsilon_t'^\top\nabla_u f_t\lambda_{t+1}}_a| &\,\leq\,\zeta\norm{\nabla J_t}^2. &&\qquad \text{(\cref{lem:cauchy})}\nonumber
    \end{alignat}
    The expression for $b$ involves $\delta_{t+1}$, which is the error of the approximate costate $\tilde\lambda_{t+1}$.
    This error comes from the cumulative error build-up due to $\varepsilon_{t+1:T-1}$, the errors of the approximate Jacobians used in the backward pass.
    To bound $|b|$ we therefore first need to bound this error build-up.
    To this end, we now show that for all $s \in [T-t]$,
    \begin{equation}\label{eq:delta}
        \norm{\delta_{t+s}} \leq \frac{\gamma}{3^{s-1}}\kappa^{-1}(\nabla_u f_t)\prod_{i=1}^{s-1}\kappa^{-1}(\nabla_x f_{t+i})\norm{\lambda_{t+s}}\text,
    \end{equation}
    where we write the inverse condition number of a matrix $A$ as $\kappa^{-1}(A) \doteq \smin(A) / \smax(A)$.
    To prove this bound, we perform a backward induction on $s$. First, consider $s = T - t$. The right-hand side of \cref{eq:delta} is clearly nonnegative. The left-hand side is
    \[
        \norm{\delta_T} = \norm{\tilde\lambda_T - \lambda_T} = 0\text,
    \]
    as $\tilde\lambda_T = \lambda_T$.
    Thus, the inequality holds for $s = T - t$.
    We now complete the induction by showing that it holds for any $s \in [T - t - 1]$, assuming that it holds for $s + 1$.
    We start by decomposing $\delta_{t+s}$:
    \begin{align*}
        \delta_{t+s} &= \tilde\lambda_{t+s} - \lambda_{t+s}\\
        &= A_{t+s} \tilde\lambda_{t+s+1} - \nabla_x f_{t+s} \lambda_{t+s+1}\\
        &= (\nabla_x f_{t+s} + \varepsilon_{t+s})(\lambda_{t+s+1} + \delta_{t+s+1}) - \nabla_x f_{t+s} \lambda_{t+s+1}\\
        &= \varepsilon_{t+s} \lambda_{t+s+1} + \nabla_x f_{t+s} \delta_{t+s+1} + \varepsilon_{t+s} \delta_{t+s+1}.
    \end{align*}
    Now, we can bound $\norm{\delta_{t+s}}$ by bounding these individual contributions:
    \[
        \norm{\delta_{t+s}} \leq \norm{\underbrace{\varepsilon_{t+s} \lambda_{t+s+1}}_{a'}} + \norm{\underbrace{\nabla_x f_{t+s} \delta_{t+s+1}}_{b'}} + \norm{\underbrace{\varepsilon_{t+s} \delta_{t+s+1}}_{c'}}.
    \]
    We start with $\norm{a'}$:
    \begin{alignat*}{3}
        &&&&\smax(\varepsilon_{t+s}) &\,\leq\,\frac{\gamma}{3^s} \kappa^{-1}(\nabla_u f_t)\prod_{i=1}^{s-1}\kappa^{-1}(\nabla_x f_{t+i})\smin(\nabla_x f_{t+s})\tag{\cref{ass:xerror}}\\
        &&\implies\quad &&\norm{\underbrace{\varepsilon_{t+s}\lambda_{t+s+1}}_{a'}} &\,\leq\,\frac{\gamma}{3^s} \kappa^{-1}(\nabla_u f_t)\prod_{i=1}^{s-1}\kappa^{-1}(\nabla_x f_{t+i})\norm{\underbrace{\nabla_x f_{t+s}\lambda_{t+s+1}}_{\lambda_{t+s}}}.\tag{\cref{lem:sigma}}
    \end{alignat*}
    Now, $\norm{b'}$:
    \begin{alignat*}{3}
        &&&&\norm{\delta_{t+s+1}} &\,\leq\,\frac{\gamma}{3^{s}}\kappa^{-1}(\nabla_u f_t)\prod_{i=1}^{s}\kappa^{-1}(\nabla_x f_{t+i})\norm{\lambda_{t+s+1}}\tag{Induction hypothesis}\\
        &&\iff\quad &&\smax(\nabla_x f_{t+s})\norm{\delta_{t+s+1}} &\,\leq\,\frac{\gamma}{3^{s}}\kappa^{-1}(\nabla_u f_t)\prod_{i=1}^{s-1}\kappa^{-1}(\nabla_x f_{t+i})\smin(\nabla_x f_{t+s})\norm{\lambda_{t+s+1}}\tag{Definition of $\kappa^{-1}$}\\
        &&\implies\quad &&\norm{\underbrace{\nabla_x f_{t+s}\delta_{t+s+1}}_{b'}} &\,\leq\,\frac{\gamma}{3^{s}}\kappa^{-1}(\nabla_u f_t)\prod_{i=1}^{s-1}\kappa^{-1}(\nabla_x f_{t+i})\norm{\underbrace{\nabla_x f_{t+s}\lambda_{t+s+1}}_{\lambda_{t+s}}}. \tag{\cref{lem:sigma}}
    \end{alignat*}
    And finally, $\norm{c'}$:
    \begin{alignat}{3}
        &&&&\smax(\varepsilon_{t+s}) &\,\leq\,\smin(\nabla_x f_{t+s})\,\leq\,\smax(\nabla_x f_{t+s})\label{eq:eps_smax}\\
        &&\implies\quad &&\smax(\varepsilon_{t+s})\norm{\delta_{t+s+1}} &\,\leq\,\frac{\gamma}{3^{s}}\kappa^{-1}(\nabla_u f_t)\prod_{i=1}^{s}\kappa^{-1}(\nabla_x f_{t+i})\smax(\nabla_x f_{t+s})\norm{\lambda_{t+s+1}}\tag{Induction hypothesis}\\
        &&\iff\quad &&\smax(\varepsilon_{t+s})\norm{\delta_{t+s+1}} &\,\leq\,\frac{\gamma}{3^{s}}\kappa^{-1}(\nabla_u f_t)\prod_{i=1}^{s-1}\kappa^{-1}(\nabla_x f_{t+i})\smin(\nabla_x f_{t+s})\norm{\lambda_{t+s+1}}\tag{Definition of $\kappa^{-1}$}\nonumber\\
        &&\implies\quad &&\norm{\underbrace{\varepsilon_{t+s}\delta_{t+s+1}}_{c'}} &\,\leq\,\frac{\gamma}{3^{s}}\kappa^{-1}(\nabla_u f_t)\prod_{i=1}^{s-1}\kappa^{-1}(\nabla_x f_{t+i})\norm{\underbrace{\nabla_x f_{t+s}\lambda_{t+s+1}}_{\lambda_{t+s}}}.\tag{\cref{lem:sigma}}\nonumber
    \end{alignat}
    Here, \cref{eq:eps_smax} follows from \cref{ass:xerror} by noting that that the constant before $\smin(\nabla_x f_{t+s})$ on the right-hand side is not greater than $1$. We can now put all three bounds together to give us \cref{eq:delta}:
    \[
        \norm{\delta_{t+s}} \leq \norm{a'} + \norm{b'} + \norm{c'} \leq 3\cdot\frac{\gamma}{3^s}\kappa^{-1}(\nabla_u f_t)\prod_{i=1}^{s-1}\kappa^{-1}(\nabla_x f_{t+i})\norm{\lambda_{t+s}}.
    \]
    
    Equipped with a bound on $\delta_{t+s}$, we are ready to bound $|b|$ and $|c|$.
    Starting with $|b|$, we have:
    \begin{alignat}{4}
        &&&&\norm{\delta_{t+1}} &\,\leq\,\gamma\kappa^{-1}(\nabla_u f_t)\norm{\lambda_{t+1}}&&\qquad \text{(Eq.~\cref{eq:delta} for $s = 1$)}\nonumber\\
        &&\iff\quad &&\smax(\nabla_u f_t)\norm{\delta_{t+1}} &\,\leq\,\gamma\smin(\nabla_u f_t)\norm{\lambda_{t+1}}&&\qquad \text{(Definition of $\kappa^{-1}$)}\nonumber\\
        &&\implies\quad &&\norm{\nabla_u f_t \delta_{t+1}} &\,\leq\,\gamma\norm{\nabla_u f_t \lambda_{t+1}}&&\qquad \text{(\cref{lem:sigma})}\label{eq:gamma}\\
        &&\implies\quad &&|\underbrace{\delta_{t+1}^\top \nabla_u f_t^\top \nabla_u f_t \lambda_{t+1}}_b| &\,\leq\,\gamma\norm{\nabla J_t}^2.&&\qquad \text{(\cref{lem:cauchy})}\nonumber
    \end{alignat}
    And finally, we can bound $|c|$:
    \begin{alignat}{4}
        &&&&\smax(\varepsilon_t') \,\leq\,\zeta\smin(\nabla_u f_t)&\,\leq\,\zeta\smax(\nabla_u f_t)&&\qquad \text{(\cref{ass:uerror})}\nonumber\\
        &&\implies\quad &&\smax(\varepsilon_t')\kappa^{-1}(\nabla_u f_t) \norm{\lambda_{t+1}}&\,\leq\,\zeta\smin(\nabla_u f_t)\norm{\lambda_{t+1}}&&\qquad \text{(Definition of $\kappa^{-1}$)}\nonumber\\
        &&\implies\quad &&\smax(\varepsilon_t')\norm{\delta_{t+1}}&\,\leq\,\gamma\zeta\smin(\nabla_u f_t)\norm{\lambda_{t+1}}&&\qquad \text{(Eq.~\cref{eq:delta} for $s = 1$)}\nonumber\\
        &&\implies\quad &&\norm{\varepsilon_t'\delta_{t+1}}&\,\leq\,\gamma\zeta\norm{\nabla_u f_t\lambda_{t+1}}&&\qquad \text{(\cref{lem:sigma})}\label{eq:gammazeta}\\
        &&\implies\quad &&|\underbrace{\delta_{t+1}^\top \varepsilon_t'^\top \nabla_u f_t \lambda_{t+1}}_c| &\,\leq\,\gamma\zeta\norm{\nabla J_t}^2.&&\qquad \text{(\cref{lem:cauchy})}\nonumber
    \end{alignat}
    This concludes the proof of the first inequality showing that
    \[
        g_t^\top \nabla J_t \geq \mu\norm{\nabla J_t}^2.
    \]
    The second inequality,
    \[
        \norm{g_t} \leq \nu\norm{\nabla J_t}\text,
    \]
    follows easily from the work we have already done. To show this, we start by decomposing $g_t$:
    \begin{align*}
        g_t &= B_t \tilde\lambda_{t+1}\\
        &= (\nabla_u f_t + \varepsilon'_t)(\lambda_{t+1} + \delta_{t+1})\\
        &= \nabla J_t + \nabla_u f_t \delta_{t+1} + \varepsilon_t' \lambda_{t+1} + \varepsilon_t' \delta_{t+1}.
    \end{align*}
    To bound the norm of $g_t$, we again make use of the triangle inequality:
    \begin{align*}
        \norm{g_t} &\leq \norm{\nabla J_t} + \norm{\nabla_u f_t \delta_{t+1}} + \norm{\varepsilon_t' \lambda_{t+1}} + \norm{\varepsilon_t' \delta_{t+1}}\\
        &\leq (1 + \gamma + \zeta + \gamma\zeta)\norm{\nabla J_t}\\
        &= \nu\norm{\nabla J_t}\text,
    \end{align*}
    where we have used \cref{eq:gamma,eq:zeta,eq:gammazeta}.
\end{proof}

\paragraph{Proof of \cref{thm:main}.}
Let $N \in \mathbb N$ and $t \in [T-1]_0$ be fixed.
In \cref{alg:model-based}, the iterates are computed, for all $k \in [N-1]_0$, as
\begin{equation*}
    u_t^{(k + 1)} = u_t^{(k)} + \eta g_t^{(k)}\text,
\end{equation*}
where $g_t^{(k)}$ is the approximate gradient at iteration $k$.
We denote the true gradient at iteration $k$ by $\nabla J_t^{(k)}$.
From the $L$-smoothness of the objective function (\cref{ass:smooth}), it follows that
\begin{alignat*}{3}
    &&J(u_{0:T-1}^{(k + 1)})&\,\geq\,J(u_{0:T-1}^{(k)}) + \nabla_{u_t}J(u_{0:T-1}^{(k)})^\top (u_t^{(k + 1)} - u_t^{(k)}) - \frac{L}{2} \norm{u_t^{(k + 1)} - u_t^{(k)}}^2&&\\
    &&&\,=\,J(u_{0:T-1}^{(k)}) + {\nabla J_t^{(k)}}^\top (\eta g_t^{(k)}) - \frac{L}{2} \norm{\eta g_t^{(k)}}^2&&\\    &&&\,\geq\, J(u_{0:T-1}^{(k)}) + \eta\mu\norm{\nabla J_t^{(k)}}^2 - \frac{\eta^2 L\nu^2}{2} \norm{\nabla J_t^{(k)}}^2&&\hspace{-5em}\text{(\cref{thm:big})}\\
    &&&\,=\,J(u_{0:T-1}^{(k)}) + \eta\bigg(\underbrace{\mu - \frac{\eta L\nu^2}{2}}_{\alpha}\bigg) \norm{\nabla J_t^{(k)}}^2.&&
\end{alignat*}
\Cref{thm:main} demands that $\eta > 0$ is set small enough such that $\alpha > 0$, which is possible because $0 < \mu < \nu$ and $L > 0$.
Thus, we get
\begin{alignat*}{4}
    &&&&\eta\alpha \norm{\nabla J_t^{(k)}}^2&\,\leq\,J(u_{0:T-1}^{(k + 1)}) - J(u_{0:T-1}^{(k)})&&\\
    &&\implies\quad&&\frac{1}{N}\sum_{k=0}^{N - 1}\norm{\nabla J_t^{(k)}}^2&\,\leq\,\frac{1}{\alpha\eta N}\sum_{k=0}^{N-1}\left\{J(u_{0:T-1}^{(k + 1)}) - J(u_{0:T-1}^{(k)})\right\}&&\\
    &&&&&\,=\,\frac{1}{\alpha\eta N}\left\{J(u_{0:T-1}^{(N)}) - J(u_{0:T-1}^{(0)})\right\}&&\\
    &&&&&\,\leq\,\frac{J^\star - J(u_{0:T-1}^{(0)})}{\alpha\eta N},&&
\end{alignat*}
where $J^\star \doteq \sup_{u\in \mathcal U^T}J(u)$ is the optimal value of the initial state.
\qed}

\iftoggle{noappendix}{
\refstepcounter{section}\label{app:pendulum}
}{
\section{Inverted pendulum swing-up task}\label{app:pendulum}
We give a brief description of the inverted pendulum system on which we evaluate our algorithms in \cref{sec:pendulum}.
The setup is shown in \cref{fig:pendulum}.
The state at time $t$ is $x_t = (\ell, \dot \ell, \theta, \dot \theta)_t \in \mathbb R^4$, where $\ell$ is the position of the cart on the bar and $\theta$ is the (signed) pendulum angle.
The action $u_t$ is the horizontal force $F$ (in units of $\SI{50}{\newton}$) applied to the cart at time $t$.
Episodes are of length $T = 100$ ($=\SI{1}{s}$), the running reward $r(x, u) = -0.001 u^2$ penalizes large forces, and the terminal reward $r_T(x) = -\norm{x}_1$ defines the goal state to be at rest in the upright position.
The system has five parameters: the mass of the cart ($m_1 = \SI{1}{kg}$), the mass of the pendulum tip ($m_2 = \SI{0.1}{kg}$), the length of the pendulum ($m_3 = \SI{0.5}{m}$), the friction coefficient for linear motion ($m_4 = \SI{0.01}{\newton\second\per\meter}$), and the friction coefficient for rotational motion ($m_5 = \SI{0.01}{\newton\meter\second\per\radian}$).
These are the model parameters that are randomly sampled to test the robustness of our model-based algorithm in \cref{sec:pendulum}.
We say that the swing-up task is \emph{solved} if $J_\mathrm{max} > -0.03$.
This threshold was determined empirically.
If the algorithm or the model is randomized, then $[J_\mathrm{max} > -0.03]$ is a Bernoulli random variable whose mean, which we call the \emph{solve rate}, depends on the quality of the learning algorithm.}

\iftoggle{noappendix}{
\refstepcounter{section}\label{app:mjc-long}
}{
\section{Further experiments}\label{app:mjc-long}
In this section, we repeat the MuJoCo experiments of \cref{sec:mjc} with longer time-horizons $T$.
The results are shown in \cref{fig:mjc-long}.
Our algorithms are sensitive to the horizon $T$ due to the backpropagation of the costates.
At each propagation step, the approximation errors of the Jacobians amplify the errors of the costates.
For this reason, \cref{thm:main} demands (through \cref{ass:error}) more accurate Jacobians at later time steps.
Thus, for large $T$, our convergence result requires more accurate Jacobian estimates.
However, in \cref{fig:mjc-long}, we see that \cref{alg:off-policy} is able to cope with longer horizons for the two MuJoCo environments.
The reason for this discrepancy between our theoretical and empirical result is that \cref{thm:main} does not consider the stability of the system under consideration.
The two MuJoCo systems, \texttt{Ant-v4} and \texttt{HalfCheetah-v4}, are stable along the trajectories encountered during training, which prevents an exponential build-up of error in the costate propagation.\looseness=-1

\begin{figure}[t]
    \centering
    \includegraphics[width=\linewidth]{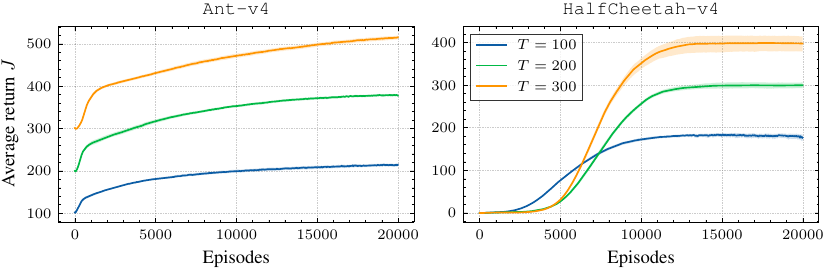}
    \caption{Learning curves of our off-trajectory algorithm. All experiments were repeated with $20$ random seeds, and we show $95\%$-bootstrap confidence intervals for the average return.}
    \label{fig:mjc-long}
    \vspace{-1em}
\end{figure}}

\iftoggle{noappendix}{
\refstepcounter{section}\label{app:baselines}
}{
\section{Baselines}\label{app:baselines}
We compare our algorithms against two baselines: the finite-difference approach discussed at the end of \cref{sec:on-policy} and the gradient-free cross-entropy method \citep[CEM;][]{rubinstein1999cross}.
These methods are listed in \cref{alg:fd,alg:cem}.
In both algorithms, we perform $M \in \mathbb N$ rollouts of perturbed action sequences $\{\bm u_i \sim \mathcal N(\bar{\bm u}, \sigma I)\}_{i=1}^M$.
Here, $\bar{\bm u}$ is the current action sequence and $\sigma > 0$ is a noise scale parameter.
In CEM, we then construct the \emph{elite} set $S$ of the $L < M$ perturbed action sequences with the highest returns, where $L \in \mathbb N$ is a hyperparameter.
Finally,  the current action sequence $\bar{\bm u}$ is updated to be the mean of the elite sequences, such that $\bar{\bm u} \gets \frac{1}{L}\sum_{\bm u \in S} \bm u$.

% This method can be considerably more efficient than the simple finite-difference method.
% Here, we are not trying to estimate the gradient anymore, so we can potentially improve the action sequence with far fewer rollouts than would be needed in the finite-difference approach.
While the gradient-free nature of this method can make it more efficient than the finite-difference approach, it still suffers from the same fundamental deficiency: it ignores the Markovian structure of the RL problem and treats the objective function $J$ as a black box.
CEM is commonly used in model-based closed-loop reinforcement learning for planning.
In this setting, the rollouts are hallucinated using the approximate model.
Instead of executing the complete open-loop trajectory, the model-predictive control framework is typically employed.
The planning procedure is repeated after each step in the real environment with the executed action being the first item in the planned action sequence.
Thus, this setting is very different from our open-loop RL objective.
For this reason, we slightly modify the CEM algorithm to better fit our requirements.
In model-based RL, typically both mean $\bar u$ and standard variation $\sigma$ are adapted in CEM \citep{hafner2019planet,pinneri2021sample}.
In our experiments, this approach led to very fast convergence ($\sigma \to 0$) to suboptimal trajectories.
We thus only fit the mean and keep the noise scale fixed, which we empirically observed to give much better results.

\begin{algorithm2e}[t]
  \caption{Finite-difference method}\label{alg:fd}
  \KwIn{Number of rollouts $M \in \mathbb N$, noise scale $\sigma > 0$, step size $\eta > 0$}
  Initialize $\bar u_{0:T-1}$ (initial action sequence)\\
  $\bar{\bm u} \gets \operatorname{vec}(\bar u_{0:T-1}) \in \mathbb R^{TK}$\\
  \For{$k = 1, 2, \dots, N$}{
      \tcp{Forward passes}
      $\bar x_{0:T} \gets \operatorname{rollout}(\bar u_{0:T-1})$\\
      $\bar J \gets \sum_{t=0}^{T-1} r(\bar x_t, \bar u_t) + r_T(\bar x_T)$\\
      \For{$i = 1, 2, \dots, M$}{
          $u_{0:T-1} \sim \mathcal N(\bar u_{0:T-1}, \sigma I)$\\
          $x_{0:T} \gets \operatorname{rollout}(u_{0:T-1})$\\
          $\bm u_i \gets \operatorname{vec}(u_{0:T-1}) \in \mathbb R^{TK}$\\
          $J_i \gets \sum_{t=0}^{T-1} r(x_t, u_t) + r_T(x_T)$
      }
      \BlankLine
      \tcp{Gradient estimation}
      $\bm g \gets \argmin_{\bm g \in \mathbb R^{TK}}\sum_{i=1}^M \{J_i - \bar J - \bm g^\top (\bm u_i - \bar{\bm u})\}^2$\\
      \BlankLine
      \tcp{Gradient ascent}
      $\bar{\bm u} \gets \bar{\bm u} + \eta \bm g$\\
      $\bar u_{0:T-1} \gets \operatorname{reshape}(\bar{\bm u}) \in \mathbb R^{T\times K}$
  }
\end{algorithm2e}

\begin{algorithm2e}[t]
  \caption{Cross-entropy method}\label{alg:cem}
  \KwIn{Number of rollouts $M \in \mathbb N$, noise scale $\sigma > 0$, size of elite set $L \in \mathbb N$}
  Initialize $\bar u_{0:T-1}$ (initial action sequence)\\
  $\bar{\bm u} \gets \operatorname{vec}(\bar u_{0:T-1}) \in \mathbb R^{TK}$\\
  \For{$k = 1, 2, \dots, N$}{
      \tcp{Forward passes}
      \For{$i = 1, 2, \dots, M$}{
          $u_{0:T-1} \sim \mathcal N(\bar u_{0:T-1}, \sigma I)$\\
          $x_{0:T} \gets \operatorname{rollout}(u_{0:T-1})$\\
          $\bm u_i \gets \operatorname{vec}(u_{0:T-1}) \in \mathbb R^{TK}$\\
          $J_i \gets \sum_{t=0}^{T-1} r(x_t, u_t) + r_T(x_T)$
      }
      \BlankLine
      \tcp{Elite set computation}
      $S \gets \operatorname{arg\ partition}_L\{(-J_i)_{i=1}^M\}_{1:L}$\\
      \BlankLine
      \tcp{Action sequence update}
      $\bar{\bm u} \gets \frac{1}{L}\sum_{i \in S} \bm u_i$\\
      $\bar u_{0:T-1} \gets \operatorname{reshape}(\bar{\bm u}) \in \mathbb R^{T\times K}$
  }
\end{algorithm2e}}

\iftoggle{noappendix}{
\refstepcounter{section}\label{app:hyper}
}{
\section{Hyperparameters}\label{app:hyper}
Unless stated otherwise, we used the hyperparameters listed in \cref{tbl:hyper-pendulum} in the inverted penulum experiments of \cref{sec:pendulum}, and those listed in \cref{tbl:hyper-mjc} in the MuJoCo experiments of \cref{sec:mjc,app:mjc-long}.
In each experiment, all actions in the initial action trajectory $u_{0:T-1}^{(0)}$ are sampled from a zero-mean Gaussian distribution with standard deviation $0.01$.
We use the Adam optimizer \citep{kingma2014adam} both for training the MLP model and for performing the gradient ascent steps in \cref{alg:model-based,alg:on-policy,alg:off-policy,alg:fd}.
We did not optimize the hyperparameters of soft actor-critic (SAC), but kept the default values suggested by \citet{haarnoja-2018a-soft}, as these are already optimized for the MuJoCo environments.
The entropy coefficient of the SAC algorithm is tuned automatically according to the procedure described by \citet{haarnoja-2018b-soft}.
In our experiments, we make use of the Stable-Baselines3 \citep{stable-baselines3} implementation of SAC.

\begin{table}
   \centering
   \begin{minipage}[t]{.49\textwidth}
     \centering
     \caption{Pendulum experiments hyperparameters}
     \begin{tabular}{lr}
       \toprule
       Parameter & Value\\ 
       \midrule 
       Number of optimization steps $N$ & $50000$\\
       Step size $\eta$ & $0.001$\\
       Noise scale $\sigma$ & $0.001$\\
       Number of perturbed rollouts $M$ & $10$\\
       Forgetting factor $\alpha$ & $0.8$\\
       Initial precision $q_0$ & $0.001$\\
       Cross-entropy method: $M$\tablefootnote{\label{foot:hyper} This value was chosen on the basis of the experiment presented in \cref{fig:rollouts}.} & $20$\\
       Finite-difference method: $M$\textsuperscript{\ref{foot:hyper}} & $20$\\
       Finite-difference method: $\sigma$\textsuperscript{\ref{foot:hyper}} & $0.0001$\\
       MLP model: hidden layers & $[16, 16]$\\
       MLP model: training rollouts & $1000$\\
       MLP model training: epochs & $10$\\
       MLP model training: batch size & $100$\\
       MLP model training: step size & $0.002$\\
       MLP model training: weight decay & $0.001$\\
       \bottomrule\\
       \end{tabular}
       \label{tbl:hyper-pendulum}
   \end{minipage}\hfill
   \begin{minipage}[t]{.49\textwidth}
     \centering
     \caption{MuJoCo experiments hyperparameters}
     \begin{tabular}{lr}
       \toprule
       Parameter & Value\\ 
       \midrule 
       Number of optimization steps $N$ & $20000$\\
       Step size $\eta$ & $0.0001$\\
       Noise scale $\sigma$ & $0.03$\\
       Initial precision $q_0$ & $0.0001$\\
       \midrule
       Forgetting factor $\alpha$ & \\
       \hspace{1em}\texttt{HalfCheetah-v4}, $T = 100$ & $0.9$\\
       \hspace{1em}\texttt{HalfCheetah-v4}, $T = 200$ & $0.8$\\
       \hspace{1em}\texttt{HalfCheetah-v4}, $T = 300$ & $0.8$\\
       \hspace{1em}\texttt{Ant-v4}, $T = 100$ & $0.95$\\
       \hspace{1em}\texttt{Ant-v4}, $T = 200$ & $0.9$\\
       \hspace{1em}\texttt{Ant-v4}, $T = 300$ & $0.85$\\
       \bottomrule\\
       \end{tabular}
       \label{tbl:hyper-mjc}
   \end{minipage}
   \vspace{-\baselineskip}
 \end{table}

For our off-trajectory method, we found it worthwile to tune the forgetting factor $\alpha$ to the specific task at hand.
Large $\alpha$ means that data is retained for longer, which both makes the algorithm more sample efficient (i.e., faster convergece) and the Jacobian estimates more biased (i.e., convergence to a worse solution).
In \cref{fig:mjc-alpha}, we show this trade-off in the learning curves for the MuJoCo tasks (with the horizon $T = 200$).
We found that the performance is much less senstitive to the choice of noise scale $\sigma$ and initial precision $q_0$ than to the choice of the forgetting factor $\alpha$.
\begin{figure}[h]
  \centering
  \includegraphics[width=\linewidth]{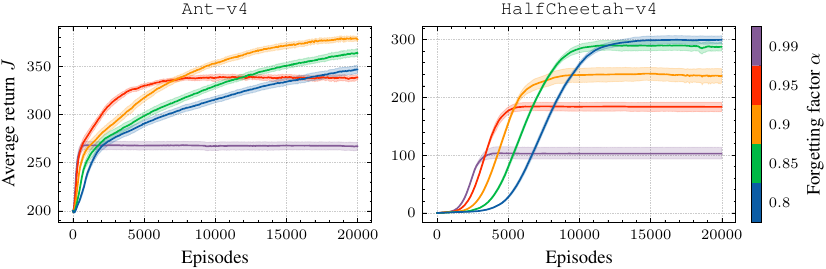}
  \caption{Analysis of the influence of the forgetting factor $\alpha$ on the performance of the off-trajectory method (\cref{alg:off-policy}) in the MuJoCo environments ($T = 200$). All experiments were repeated with $20$ random seeds, and we show $95\%$-bootstrap confidence intervals for the average return.}
  \label{fig:mjc-alpha}
  \vspace{-\baselineskip}
\end{figure}}
\end{document}

%% file: loop.tikz
\definecolor{neptune3}{rgb}{.53, .85, .95}%
\newcommand{\earth}{
    \clip (0, 0) circle (1);

    % Globe
    \draw [fill=cyan] (0,0) circle (1);
    
    % Continents
    \draw [very thin, green, fill=green!70!black] (0.6, 0.4) circle (.4 and .2) ;
    \draw [very thin, green, fill=green!70!black] (0.6, -.2) circle (.3 and .4) ;
    
    \draw [very thin, green, fill=green!70!black] (-.7, 0.4) circle (.25 and .2) ;
    \draw [very thin, green, fill=green!70!black] (-.6, -.2) circle (.25 and .4) ;
    
    % Poles (thank you, law of cosines!)
    \fill [white] (0.38268343236508984, 0.9238795325112867) arc (67.5 : 112.5 : 1) arc (-(90 + 22.06219115754148 / 2) : -(90 - 22.06219115754148 / 2) : 2) -- cycle ;
    \fill [white] (0.38268343236508984, -0.9238795325112867) arc (-67.5 : -112.5 : 1) arc ((90 + 22.06219115754148 / 2) : (90 - 22.06219115754148 / 2) : 2) -- cycle ;
    
    % Edge 
    \draw [neptune3, ultra thick] (0,0) circle (1);
}%
\begin{tikzpicture}[line cap=round]
    \node[circle] (env) {$\phantom{f(x_t, u_t)}$};
    \node[draw, rectangle, inner sep=4, above=2 of env.center, anchor=center, rounded corners=5pt] (pi) {$\pi(\,\cdot\mid x_t)$};
    % \node[above=0.2 of pi, scale=0.9] (text) {Closed-loop control};
    % \planet[surface=earth, scale=.81, center={(env)}];
    \begin{scope}[shift={(env)}, scale=.88]
        \earth
    \end{scope}
    \node[draw, circle, fill=white, fill opacity=0.8, text opacity=1] (f) at (env) {$f(x_t, u_t)$};
    \draw[->] (pi.south) -- node[right]{$u_t$} (env.north);
    \node[below=0.5 of env] (dash) {};
    \draw[densely dotted] ($(dash) + (0, 0.25)$) -- ($(dash) - (0, 0.25)$);
    \draw[->] (env.east) -- node[above]{$x_{t+1}$} ++(1, 0) |- (dash);
    \draw[->] (dash) -| ($(env.west) - (1, 0)$) -- node[above]{$x_t$} (env.west);
    \path let \p1 = (pi), \p2 = ($(env.west) - (1, 0)$) in node[inner sep=0, outer sep=0] (p) at (\x2,\y1) {};
    \draw[->] ($(env.west) - (1, 0)$) |- (p.center) -- node[above]{$x_t$} (pi.west);

    \node[circle, right=4 of env.center, anchor=center] (env2) {$\phantom{f(x_t, u_t)}$};
    \node[draw, rectangle, inner sep=4, right=4 of pi.center, anchor=center, rounded corners=5pt] (pi2) {$\phantom{\pi(i\mid i)}$};
    \node[anchor=base] at (pi2.base) {$u_{0:T-1}$};
    % \node[right=4 of text.center, anchor=center, scale=0.9] (text2) {Open-loop control};
    \begin{scope}[shift={(env2)}, scale=.88]
        \earth
    \end{scope}
    \node[draw, circle, fill=white, fill opacity=0.8, text opacity=1] (f2) at (env2) {$f(x_t, u_t)$};
    \draw[->] (pi2.south) -- node[right]{$u_t$} (env2.north);
    \node[below=0.5 of env2] (dash2) {};
    \draw[densely dotted] ($(dash2) + (0, 0.25)$) -- ($(dash2) - (0, 0.25)$);
    \draw[->] (env2.east) -- node[above]{$x_{t+1}$} ++(1, 0) |- (dash2);
    \draw[->] (dash2) -| ($(env2.west) - (1, 0)$) -- node[above]{$x_t$} (env2.west);
\end{tikzpicture}%

%% file: main.bbl
\begin{thebibliography}{48}
\providecommand{\natexlab}[1]{#1}
\providecommand{\url}[1]{\texttt{#1}}
\expandafter\ifx\csname urlstyle\endcsname\relax
  \providecommand{\doi}[1]{doi: #1}\else
  \providecommand{\doi}{doi: \begingroup \urlstyle{rm}\Url}\fi

\bibitem[{\AA}str{\"o}m and H{\"a}gglund(1995)]{astrom1995pid}
Karl~J. {\AA}str{\"o}m and Tore H{\"a}gglund.
\newblock \emph{PID Controllers: Theory, Design, and Tuning}.
\newblock International Society for Measurement and Control, second edition,
  1995.

\bibitem[{\AA}str{\"o}m and Murray(2021)]{aastrom2021feedback}
Karl~J. {\AA}str{\"o}m and Richard~M. Murray.
\newblock \emph{Feedback Systems: An Introduction for Scientists and
  Engineers}.
\newblock Princeton University Press, second edition, 2021.

\bibitem[Betts(2010)]{betts2010practical}
John~T. Betts.
\newblock \emph{Practical Methods for Optimal Control and Estimation Using
  Nonlinear Programming}.
\newblock SIAM, 2010.

\bibitem[B{\"o}ttcher et~al.(2022)B{\"o}ttcher, Antulov-Fantulin, and
  Asikis]{bottcher2022ai}
Lucas B{\"o}ttcher, Nino Antulov-Fantulin, and Thomas Asikis.
\newblock {AI} {P}ontryagin or how artificial neural networks learn to control
  dynamical systems.
\newblock \emph{Nature communications}, 13\penalty0 (333), 2022.
\newblock URL \url{https://doi.org/10.1038/s41467-021-27590-0}.

\bibitem[Bottou et~al.(2018)Bottou, Curtis, and
  Nocedal]{bottou2018optimization}
L{\'e}on Bottou, Frank~E Curtis, and Jorge Nocedal.
\newblock Optimization methods for large-scale machine learning.
\newblock \emph{SIAM review}, 60\penalty0 (2):\penalty0 223--311, 2018.
\newblock URL \url{https://doi.org/10.1137/16M1080173}.

\bibitem[Brockman et~al.(2016)Brockman, Cheung, Pettersson, Schneider,
  Schulman, Tang, and Zaremba]{brockman-2016-gym}
Greg Brockman, Vicki Cheung, Ludwig Pettersson, Jonas Schneider, John Schulman,
  Jie Tang, and Wojciech Zaremba.
\newblock {OpenAI} {Gym}.
\newblock \emph{arXiv:1606.01540}, 2016.
\newblock URL \url{http://arxiv.org/abs/1606.01540}.

\bibitem[Carraro et~al.(2015)Carraro, Geiger, K{\"o}rkel, and
  Rannacher]{Carraro2015}
Thomas Carraro, Michael Geiger, Stefan K{\"o}rkel, and Rolf Rannacher, editors.
\newblock \emph{Multiple Shooting and Time Domain Decomposition Methods}.
\newblock Springer, 2015.

\bibitem[Chen and Braun(2019)]{chen-2019-hardware}
Yuqing Chen and David~J. Braun.
\newblock Hardware-in-the-loop iterative optimal feedback control without
  model-based future prediction.
\newblock \emph{IEEE Transactions on Robotics}, 35\penalty0 (6):\penalty0
  1419--1434, 2019.
\newblock URL \url{https://doi.org/10.1109/TRO.2019.2929014}.

\bibitem[Degrave et~al.(2022)Degrave, Felici, Buchli, Neunert, Tracey,
  Carpanese, Ewalds, Hafner, Abdolmaleki, de~Las~Casas,
  et~al.]{degrave2022magnetic}
Jonas Degrave, Federico Felici, Jonas Buchli, Michael Neunert, Brendan Tracey,
  Francesco Carpanese, Timo Ewalds, Roland Hafner, Abbas Abdolmaleki, Diego
  de~Las~Casas, et~al.
\newblock Magnetic control of tokamak plasmas through deep reinforcement
  learning.
\newblock \emph{Nature}, 602\penalty0 (7897):\penalty0 414--419, 2022.
\newblock URL \url{https://doi.org/10.1038/s41586-021-04301-9}.

\bibitem[Diehl et~al.(2006)Diehl, Bock, Diedam, and Wieber]{diehl2006}
Moritz Diehl, Hans~Georg Bock, Holger Diedam, and Pierre-Brice Wieber.
\newblock Fast direct multiple shooting algorithms for optimal robot control.
\newblock In \emph{Fast Motions in Biomechanics and Robotics: Optimization and
  Feedback Control}, pages 65--93. Springer, 2006.
\newblock URL \url{https://doi.org/10.1007/978-3-540-36119-0_4}.

\bibitem[Eberhard et~al.(2023)Eberhard, Hollenstein, Pinneri, and
  Martius]{eberhard-2023-pink}
Onno Eberhard, Jakob Hollenstein, Cristina Pinneri, and Georg Martius.
\newblock Pink noise is all you need: Colored noise exploration in deep
  reinforcement learning.
\newblock In \emph{Proceedings of the Eleventh International Conference on
  Learning Representations}, 2023.
\newblock URL \url{https://openreview.net/forum?id=hQ9V5QN27eS}.

\bibitem[Fragkiadaki et~al.(2015)Fragkiadaki, Levine, Felsen, and
  Malik]{fragkiadaki2015recurrent}
Katerina Fragkiadaki, Sergey Levine, Panna Felsen, and Jitendra Malik.
\newblock Recurrent network models for human dynamics.
\newblock In \emph{Proceedings of the IEEE International Conference on Computer
  Vision}, pages 4346--4354, 2015.
\newblock URL \url{https://doi.org/10.1109/ICCV.2015.488}.

\bibitem[Gamiz et~al.(2020)Gamiz, Martínez, Grau, Bolea, and
  Vilanova]{gamiz-2020-feed}
Javier Gamiz, Herminio Martínez, Antoni Grau, Yolanda Bolea, and Ramón
  Vilanova.
\newblock Feed-forward control for a drinking water treatment plant
  chlorination process.
\newblock In \emph{2020 25th IEEE International Conference on Emerging
  Technologies and Factory Automation (ETFA)}, volume~1, pages 462--467, 2020.
\newblock URL \url{https://doi.org/10.1109/ETFA46521.2020.9211884}.

\bibitem[Geering(2007)]{geering2007}
Hans~P. Geering.
\newblock \emph{Optimal Control with Engineering Applications}.
\newblock Springer, 2007.

\bibitem[Haarnoja et~al.(2018{\natexlab{a}})Haarnoja, Zhou, Abbeel, and
  Levine]{haarnoja-2018a-soft}
Tuomas Haarnoja, Aurick Zhou, Pieter Abbeel, and Sergey Levine.
\newblock Soft actor-critic: Off-policy maximum entropy deep reinforcement
  learning with a stochastic actor.
\newblock In \emph{Proceedings of the 35th International Conference on Machine
  Learning}, volume~80 of \emph{Proceedings of Machine Learning Research},
  pages 1856--1865. PMLR, 2018{\natexlab{a}}.
\newblock URL \url{http://proceedings.mlr.press/v80/haarnoja18b.html}.

\bibitem[Haarnoja et~al.(2018{\natexlab{b}})Haarnoja, Zhou, Hartikainen,
  Tucker, Ha, Tan, Kumar, Zhu, Gupta, Abbeel, and Levine]{haarnoja-2018b-soft}
Tuomas Haarnoja, Aurick Zhou, Kristian Hartikainen, George Tucker, Sehoon Ha,
  Jie Tan, Vikash Kumar, Henry Zhu, Abhishek Gupta, Pieter Abbeel, and Sergey
  Levine.
\newblock Soft actor-critic algorithms and applications.
\newblock \emph{arXiv:1812.05905}, 2018{\natexlab{b}}.
\newblock URL \url{http://arxiv.org/abs/1812.05905}.

\bibitem[Hafner et~al.(2019)Hafner, Lillicrap, Fischer, Villegas, Ha, Lee, and
  Davidson]{hafner2019planet}
Danijar Hafner, Timothy Lillicrap, Ian Fischer, Ruben Villegas, David Ha,
  Honglak Lee, and James Davidson.
\newblock Learning latent dynamics for planning from pixels.
\newblock In \emph{Proceedings of the 36th International Conference on Machine
  Learning}, volume~97 of \emph{Proceedings of Machine Learning Research},
  pages 2555--2565. PMLR, 2019.
\newblock URL \url{https://proceedings.mlr.press/v97/hafner19a.html}.

\bibitem[Hansen et~al.(1996)Hansen, Barto, and
  Zilberstein]{hansen1996reinforcement}
Eric Hansen, Andrew Barto, and Shlomo Zilberstein.
\newblock Reinforcement learning for mixed open-loop and closed-loop control.
\newblock \emph{Advances in Neural Information Processing Systems}, 9, 1996.
\newblock URL
  \url{https://proceedings.neurips.cc/paper/1996/hash/ab1a4d0dd4d48a2ba1077c4494791306-Abstract.html}.

\bibitem[Horn et~al.(2013)Horn, Gros, and Diehl]{horn-2013-numerical}
Greg Horn, S{\'e}bastien Gros, and Moritz Diehl.
\newblock Numerical trajectory optimization for airborne wind energy systems
  described by high fidelity aircraft models.
\newblock In Uwe Ahrens, Moritz Diehl, and Roland Schmehl, editors,
  \emph{Airborne Wind Energy}, pages 205--218. Springer, 2013.
\newblock URL \url{https://doi.org/10.1007/978-3-642-39965-7_11}.

\bibitem[Howe et~al.(2022)Howe, Dufort-Labb{\'e}, Rajkumar, and
  Bacon]{howe2022myriad}
Nikolaus Howe, Simon Dufort-Labb{\'e}, Nitarshan Rajkumar, and Pierre-Luc
  Bacon.
\newblock Myriad: a real-world testbed to bridge trajectory optimization and
  deep learning.
\newblock In \emph{Advances in Neural Information Processing Systems},
  volume~35, pages 29801--29815, 2022.
\newblock URL
  \url{https://proceedings.neurips.cc/paper_files/paper/2022/hash/c0b91f9a3587bf35287f41dba5d20233-Abstract-Datasets_and_Benchmarks.html}.

\bibitem[Janjua et~al.(2024)Janjua, Shah, White, Miahi, Machado, and
  White]{janjua2024gvfs}
Muhammad~Kamran Janjua, Haseeb Shah, Martha White, Erfan Miahi, Marlos~C
  Machado, and Adam White.
\newblock {GVF}s in the real world: making predictions online for water
  treatment.
\newblock \emph{Machine Learning}, 113\penalty0 (8):\penalty0 5151--5181, 2024.
\newblock URL \url{https://doi.org/10.1007/s10994-023-06413-x}.

\bibitem[Jin et~al.(2020)Jin, Wang, Yang, and Mou]{jin2020pontryagin}
Wanxin Jin, Zhaoran Wang, Zhuoran Yang, and Shaoshuai Mou.
\newblock Pontryagin differentiable programming: An end-to-end learning and
  control framework.
\newblock In \emph{Advances in Neural Information Processing Systems},
  volume~33, pages 7979--7992, 2020.
\newblock URL
  \url{https://proceedings.neurips.cc/paper/2020/hash/5a7b238ba0f6502e5d6be14424b20ded-Abstract.html}.

\bibitem[Kingma and Ba(2014)]{kingma2014adam}
Diederik~P. Kingma and Jimmy Ba.
\newblock Adam: A method for stochastic optimization.
\newblock In \emph{Proceedings of the Third International Conference on
  Learning Representations}, 2014.
\newblock URL \url{http://arxiv.org/abs/1412.6980}.

\bibitem[Lee et~al.(2016)Lee, Delbruck, and Pfeiffer]{lee-2016-training}
Jun~Haeng Lee, Tobi Delbruck, and Michael Pfeiffer.
\newblock Training deep spiking neural networks using backpropagation.
\newblock \emph{Frontiers in Neuroscience}, 10, 2016.
\newblock URL \url{https://doi.org/10.3389/fnins.2016.00508}.

\bibitem[Ljung(1999)]{ljung1999system}
Lennart Ljung.
\newblock \emph{System Identification: Theory for the User}.
\newblock Prentice Hall, 1999.

\bibitem[Ma et~al.(2022)Ma, B\"{u}chler, Sch\"{o}lkopf, and Muehlebach]{ma2022}
Hao Ma, Dieter B\"{u}chler, Bernhard Sch\"{o}lkopf, and Michael Muehlebach.
\newblock A learning-based iterative control framework for controlling a robot
  arm with pneumatic artificial muscles.
\newblock In \emph{Proceedings of Robotics: Science and Systems}, 2022.
\newblock URL \url{https://www.roboticsproceedings.org/rss18/p029.html}.

\bibitem[Ma et~al.(2023)Ma, B\"{u}chler, Sch\"{o}lkopf, and Muehlebach]{ma2023}
Hao Ma, Dieter B\"{u}chler, Bernhard Sch\"{o}lkopf, and Michael Muehlebach.
\newblock Reinforcement learning with model-based feedforward inputs for
  robotic table tennis.
\newblock \emph{Autonomous Robots}, 47:\penalty0 1387--1403, 2023.
\newblock URL \url{https://doi.org/10.1007/s10514-023-10140-6}.

\bibitem[Mattei et~al.(2006)Mattei, Albanese, Ambrosino, and
  Portone]{mattei-2006-open}
Massimiliano Mattei, Raffaele Albanese, Giuseppe Ambrosino, and Alfredo
  Portone.
\newblock Open loop control strategies for plasma scenarios: Linear and
  nonlinear techniques for configuration transitions.
\newblock In \emph{Proceedings of the 45th IEEE Conference on Decision and
  Control}, pages 2220--2225, 2006.
\newblock URL \url{https://doi.org/10.1109/CDC.2006.377412}.

\bibitem[Meerkov(1980)]{meerkov-1980-principle}
Semyon~M. Meerkov.
\newblock Principle of vibrational control: Theory and applications.
\newblock \emph{IEEE Transactions on Automatic Control}, 25\penalty0
  (4):\penalty0 755--762, 1980.
\newblock URL \url{https://doi.org/10.1109/TAC.1980.1102426}.

\bibitem[Moore(1993)]{moore1993}
Kevin~L. Moore.
\newblock \emph{Iterative Learning Control for Deterministic Systems}.
\newblock Springer, 1993.
\newblock URL \url{https://doi.org/10.1007/978-1-4471-1912-8}.

\bibitem[Niemeyer et~al.(2020)Niemeyer, Mescheder, Oechsle, and
  Geiger]{Niemeyer2020CVPR}
Michael Niemeyer, Lars Mescheder, Michael Oechsle, and Andreas Geiger.
\newblock Differentiable volumetric rendering: Learning implicit {3D}
  representations without {3D} supervision.
\newblock In \emph{Proceedings of the {IEEE/CVF} Conference on Computer Vision
  and Pattern Recognition}, pages 3504--3515, 2020.
\newblock URL
  \url{https://openaccess.thecvf.com/content\_CVPR\_2020/html/Niemeyer\_Differentiable\_Volumetric\_Rendering\_Learning\_Implicit\_3D\_Representations\_Without\_3D\_Supervision\_CVPR\_2020\_paper.html}.

\bibitem[Pinneri et~al.(2021)Pinneri, Sawant, Blaes, Achterhold, Stueckler,
  Rolinek, and Martius]{pinneri2021sample}
Cristina Pinneri, Shambhuraj Sawant, Sebastian Blaes, Jan Achterhold, Joerg
  Stueckler, Michal Rolinek, and Georg Martius.
\newblock Sample-efficient cross-entropy method for real-time planning.
\newblock In \emph{Proceedings of the 2020 Conference on Robot Learning},
  volume 155 of \emph{Proceedings of Machine Learning Research}, pages
  1049--1065. PMLR, 2021.
\newblock URL \url{https://proceedings.mlr.press/v155/pinneri21a.html}.

\bibitem[Pontryagin et~al.(1962)Pontryagin, Boltayanskii, Gamkrelidze, and
  Mishchenko]{pontryagin-1962-dynamic}
Lev~S. Pontryagin, Vladimir~G. Boltayanskii, Revaz~V. Gamkrelidze, and
  Evgenii~F. Mishchenko.
\newblock \emph{Mathematical Theory of Optimal Processes}.
\newblock Wiley, 1962.

\bibitem[Raffin et~al.(2021)Raffin, Hill, Gleave, Kanervisto, Ernestus, and
  Dormann]{stable-baselines3}
Antonin Raffin, Ashley Hill, Adam Gleave, Anssi Kanervisto, Maximilian
  Ernestus, and Noah Dormann.
\newblock {Stable-Baselines3}: Reliable reinforcement learning implementations.
\newblock \emph{Journal of Machine Learning Research}, 22\penalty0
  (268):\penalty0 1--8, 2021.
\newblock URL \url{http://jmlr.org/papers/v22/20-1364.html}.

\bibitem[Raffin et~al.(2023)Raffin, Sigaud, Kober, Albu-Sch{\"a}ffer,
  Silv{\'e}rio, and Stulp]{raffin2023simple}
Antonin Raffin, Olivier Sigaud, Jens Kober, Alin Albu-Sch{\"a}ffer, Jo{\~a}o
  Silv{\'e}rio, and Freek Stulp.
\newblock A simple open-loop baseline for reinforcement learning locomotion
  tasks.
\newblock \emph{arXiv:2310.05808}, 2023.

\bibitem[Recht(2019)]{recht2019}
Benjamin Recht.
\newblock A tour of reinforcement learning: The view from continuous control.
\newblock \emph{Annual Review of Control, Robotics, and Autonomous Systems},
  2:\penalty0 253--279, 2019.
\newblock URL \url{https://doi.org/10.1146/annurev-control-053018-023825}.

\bibitem[Rubinstein(1999)]{rubinstein1999cross}
Reuven Rubinstein.
\newblock The cross-entropy method for combinatorial and continuous
  optimization.
\newblock \emph{Methodology and Computing in Applied Probability}, 1:\penalty0
  127--190, 1999.
\newblock URL \url{https://doi.org/10.1023/A:1010091220143}.

\bibitem[Schaal and Atkeson(2010)]{schaal2010}
Stefan Schaal and Christopher~G. Atkeson.
\newblock Learning control in robotics.
\newblock \emph{IEEE Robotics \& Automation Magazine}, 17\penalty0
  (2):\penalty0 20--29, 2010.
\newblock URL \url{https://doi.org/10.1109/MRA.2010.936957}.

\bibitem[Sch\"{a}fer et~al.(2007)Sch\"{a}fer, K\"{u}hl, Diehl, Schl\"{o}der,
  and Bock]{diehl2007}
Andreas Sch\"{a}fer, Peter K\"{u}hl, Moritz Diehl, Johannes Schl\"{o}der, and
  Hans~Georg Bock.
\newblock Fast reduced multiple shooting methods for nonlinear model predictive
  control.
\newblock \emph{Chemical Engineering and Processing: Process Intensification},
  46\penalty0 (11):\penalty0 1200--1214, 2007.
\newblock URL \url{https://doi.org/10.1016/j.cep.2006.06.024}.

\bibitem[Sferrazza et~al.(2020)Sferrazza, Muehlebach, and
  D'Andrea]{sferrazza2020}
Carmelo Sferrazza, Michael Muehlebach, and Raffaello D'Andrea.
\newblock Learning-based parametrized model predictive control for trajectory
  tracking.
\newblock \emph{Optimal Control Applications and Methods}, 41\penalty0
  (6):\penalty0 2225--2249, 2020.
\newblock URL \url{https://doi.org/10.1002/oca.2656}.

\bibitem[Skogestad and Postlethwaite(2005)]{skogestad2005multivariable}
Sigurd Skogestad and Ian Postlethwaite.
\newblock \emph{Multivariable Feedback Control: Analysis and Design}.
\newblock Wiley, second edition, 2005.

\bibitem[Sutton and Barto(2018)]{sutton2018reinforcement}
Richard~S. Sutton and Andrew~G. Barto.
\newblock \emph{Reinforcement Learning: An Introduction}.
\newblock MIT press, second edition, 2018.

\bibitem[Todorov et~al.(2012)Todorov, Erez, and Tassa]{todorov2012mujoco}
Emanuel Todorov, Tom Erez, and Yuval Tassa.
\newblock {MuJoCo}: A physics engine for model-based control.
\newblock In \emph{Proceedings of the {IEEE/RSJ} International Conference on
  Intelligent Robots and Systems}, pages 5026--5033, 2012.
\newblock URL \url{https://doi.org/10.1109/IROS.2012.6386109}.

\bibitem[Towers et~al.(2023)Towers, Terry, Kwiatkowski, Balis, Cola, Deleu,
  Goulão, Kallinteris, KG, Krimmel, et~al.]{towers-2023-gymnaisum}
Mark Towers, Jordan~K. Terry, Ariel Kwiatkowski, John~U. Balis, Gianluca~de
  Cola, Tristan Deleu, Manuel Goulão, Andreas Kallinteris, Arjun KG, Markus
  Krimmel, et~al.
\newblock Gymnasium, 2023.
\newblock URL \url{https://github.com/Farama-Foundation/Gymnasium}.

\bibitem[van Zundert and Oomen(2018)]{Zundert2018}
Jurgen van Zundert and Tom Oomen.
\newblock On inversion-based approaches for feedforward and {ILC}.
\newblock \emph{Mechatronics}, 50:\penalty0 282--291, 2018.
\newblock URL \url{https://doi.org/10.1016/j.mechatronics.2017.09.010}.

\bibitem[Verscheure et~al.(2009)Verscheure, Demeulenaere, Swevers, De~Schutter,
  and Diehl]{verscheure-2009-time}
Diederik Verscheure, Bram Demeulenaere, Jan Swevers, Joris De~Schutter, and
  Moritz Diehl.
\newblock Time-optimal path tracking for robots: A convex optimization
  approach.
\newblock \emph{IEEE Transactions on Automatic Control}, 54\penalty0
  (10):\penalty0 2318--2327, 2009.
\newblock URL \url{https://doi.org/10.1109/TAC.2009.2028959}.

\bibitem[Vlastelica et~al.(2020)Vlastelica, Paulus, Musil, Martius, and
  Rolínek]{pogancic2020}
Marin Vlastelica, Anselm Paulus, Vít Musil, Georg Martius, and Michal
  Rolínek.
\newblock Differentiation of blackbox combinatorial solvers.
\newblock In \emph{Proceedings of the Eighth International Conference on
  Learning Representations}, 2020.
\newblock URL \url{https://openreview.net/forum?id=BkevoJSYPB}.

\bibitem[Wright et~al.(2022)Wright, Onodera, Stein, Wang, Schachter, Hu, and
  McMahon]{Wright2022}
Logan~G. Wright, Tatsuhiro Onodera, Martin~M. Stein, Tianyu Wang, Darren~T.
  Schachter, Zoey Hu, and Peter~L. McMahon.
\newblock Deep physical neural networks trained with backpropagation.
\newblock \emph{Nature}, 601\penalty0 (7894):\penalty0 549--555, 2022.
\newblock URL \url{https://doi.org/10.1038/s41586-021-04223-6}.

\end{thebibliography}
